\documentclass{amsart}%
\usepackage[T1]{fontenc}
\usepackage{amsfonts}
\usepackage{svg}
\usepackage{mathtools,graphicx}
\usepackage{longtable}
\usepackage{array}
\usepackage{latexsym}
\usepackage{hyperref}
\usepackage{amsmath}
\usepackage{mathrsfs}
\usepackage{pstricks}
\usepackage{listings}
\usepackage{hyperref}
\usepackage{enumitem}
\usepackage{csvsimple}
\usepackage{todonotes}
\usepackage{amssymb,bm}
\usepackage{amsmath}
\usepackage{amsthm}
\usepackage{tikz}
\usepackage{tikz-cd}
\usepackage{epsfig}
\usepackage{verbatim}
\usepackage{float}
\usepackage{tabularx}
\usepackage{algorithmic}
\usepackage{amssymb}
\usepackage{graphicx}%
\providecommand{\U}[1]{\protect\rule{.1in}{.1in}}
\definecolor{codegray}{rgb}{0.5,0.5,0.5}
\lstdefinestyle{mystyle}{
commentstyle=\color{codegreen},
keywordstyle=\color{magenta},
numberstyle=\tiny\color{codegray},
stringstyle=\color{codepurple},
basicstyle=\ttfamily\footnotesize,
breakatwhitespace=false,
breaklines=true,
captionpos=b,
keepspaces=true,
numbers=left,
numbersep=5pt,
showspaces=false,
showstringspaces=false,
showtabs=false,
tabsize=2
}
\lstdefinelanguage{GAP}{
basicstyle=\ttfamily,
keywords={true, false, function, return, fail, if, in, while, do, od, else, elif, fi, break, continue},
keywordstyle=\color{blue}\bfseries,
otherkeywords={      >, <, ==
},
breaklines=true,
identifierstyle=\color{black},
sensitive=True,
comment=[l]{\#},
commentstyle=\color{cyan},
stringstyle=\color{red},
morestring=[b]',
morestring=[b]"
}
\providecommand{\U}[1]{\protect\rule{.1in}{.1in}}

\newcolumntype{Y}{>{\raggedleft\arraybackslash}X}

\def\R{{\mathbb{R}}}

\def\vs{\vskip.3cm}

\newtheorem{theorem}{Theorem}

\newtheorem{corollary}[theorem]{Corollary}

\newtheorem{definition}[theorem]{Definition}

\newtheorem{lemma}[theorem]{Lemma}

\newtheorem{proposition}[theorem]{Proposition}
\newtheorem{remark}[theorem]{Remark}

\def\vs{\vskip.3cm}

\def\R{{\mathbb{R}}}

\def\vs{\vskip.3cm}

\def\cV{\mathcal V}

\def\id{\text{\rm Id}}

\begin{document}
\title[Topological Symmetry Index]{A General Topological Symmetry Index for Critical Orbits In Neural Network Loss Landscapes}
\author[J. Liu]{Jingzhou Liu}
\address{Department of Mathematical Sciences, Texas Christian University
%Richardson, 75080 USA
}
\address{Texas Christian University}
\email{Jingzhou.Liu@UTDallas.edu}

\begin{abstract}
A systematic understanding of the global structure of neural-network loss landscapes remains challenging due to their high dimensionality. Symmetries arising from network architectures and learning objectives provide a way to investigate aspects of this structure. Although a growing number of studies have drawn attention to these symmetries, %and discussed their implications for loss minima, training dynamics and \textcolor{red}{internal representations}
a systematic characterization of them, as well as of the critical orbits they give rise to, remains incomplete. To address this gap, in this work  we establish a topological symmetry index, based on the methodologies from the family of equivariant degree theories, to decode these symmetries. The frameworks apply to arbitrary network architectures and include methods applicable to degenerate critical orbits. The index %, globally or within a prescribed region of the loss landscape,
 detects the isotropy of the critical points and quantifies the corresponding net topological contributions. Several examples, including softplus MLPs, Transformers, and equivariant neural networks, are presented to illustrate the computations.
\end{abstract}\vs
\maketitle

%\medskip


\section{Introduction}
Symmetries in neural network give rise to a rich variety of phenomena in its high-dimensional loss landscapes. One line of research, developed by Arjevani and Field, studies the spurious minima and their symmetries in shallow networks, as well as their bifurcations as network width $k$ varies. For two-layer ReLU networks in a specific teacher–student setting, they observed in \cite{arjevani2019principle} that critical points obtained by SGD have symmetries conjugate to maximal subgroups of $S_k$ of intransitive type. In \cite{arjevani2021symmetry}, they analytically constructed these families of critical points and expressed them as convergent power series in $k^{-1/2}.$ The corresponding Hessian spectra were analyzed in \cite{arjevani2020analytic}. %, yielding eigenvalue multiplicities and asymptotic expansions as $k\to\infty$. 
They further extended analogous results to symmetric tensor decomposition in \cite{arjevani2021symmetry2}, and studied the optimization landscape of networks with polynomial activations \cite{arjevani2026geometry}. In \cite{arjevani2021analytic}, they identified width thresholds in a shallow ReLU network at which saddles turn into local minima and  rigorously established in \cite{arjevani2022annihilation} that over-parameterization (adding neurons) introduces negative Hessian eigenvalues and thereby turns local minima into saddles. Using equivariant bifurcation theory, they further showed in \cite{arjevani2022equivariant} that the local minima emerge through steady-state bifurcations along the standard representation of $S_k$ rather than through spontaneous symmetry breaking of global minima. Correspondingly, Liu \cite{liu2025leaky} studied bifurcations from global minima in leaky ReLU networks and obtained unified results on the bifurcating branches and their symmetries.\vskip 0.3cm
\noindent Beyond the structure of critical points, other work connects network symmetries to the broader geometry of loss landscapes, training dynamics, and optimization. Zhao et al. \cite{zhao2022symmetries} identified flat low-loss valleys with continuous symmetry-induced critical orbits, and derived initialization-fixed conserved quantities that constrain gradient-flow reachability. In \cite{zhao2025understanding}, they further studied the connectivity of minima and identified loss barriers along linear mode connectivity. Zamir et al. \cite{zamir2025improving} improved the learning-to-optimize (L2O) framework by enabling the optimizer to select, along group orbits, parameters that are more favorable for subsequent descent prior to local updates. They also address the discovery of symmetry in \cite{zhao2026discovering}.\vskip0.3cm
\noindent The use of symmetry extends further to network design, where equivariant architectures incorporate prescribed symmetries to improve data efficiency and physical consistency. Kondor and Trivedi \cite{kondor2018generalization} established, under appropriate assumptions, the equivalence between compact-group equivariance and generalized convolutional structure. Tensor field networks \cite{thomas2018tensor} combine spherical harmonics and Clebsch–Gordan coupling to construct equivariant layers for three-dimensional point clouds, while \cite{kondor2018clebsch} use tensor products and irreducible decompositions to implement equivariant nonlinearities entirely in Fourier space. \cite{fox2022concentric} developed concentric spherical networks combining intra-sphere graph convolutions with inter-sphere radial convolutions, with equivariance established for the icosahedral rotation group. The e3nn framework \cite{geiger2022e3nn} provides a unified set of composable operations for constructing \(E(3)\)-equivariant architectures. For physical dynamical systems, Wang et al. \cite{wang2021incorporating} designed forecasting networks incorporating different physical symmetries, and demonstrated improvements in prediction and generalization. They subsequently introduced relaxed group and steerable convolutions with symmetry regularization to accommodate imperfectly symmetric dynamics \cite{wang2022approximately}.
\vskip 0.3cm
These studies demonstrate the importance of symmetry for understanding loss landscapes and training dynamics, as well as its value for optimizer and network design. %They also yield precise information in the settings they address. 
However, when it comes to the symmetric structure of the loss landscape, most of the above analyses either focus on a specific network architecture or presuppose \textit{a priori} knowledge of the critical points. %For example, Arjevani and Field obtain detailed results through constructions and spectral analyses developed for specific models, while Zhao and collaborators take given minima as the starting point for their symmetry-orbit analysis. 
These considerations motivate a complementary question:\begin{quote}
\itshape Can a unified, computable criterion certify the presence of critical orbits and characterize their symmetries
across different network architectures without explicitly constructing the
critical points?
\end{quote}\vskip.3cm
In this work, we address this question by introducing a topological symmetry index based on the methodology of Equivariant Degree Theory. Under suitable hypotheses, the index provides criteria for the existence of critical orbits with specified symmetries within a prescribed region. \vskip0.3cm%This information need not persist as the region expands, since contributions to the index may cancel, obscuring symmetry information detected in smaller regions. We present the underlying theory in Sections \ref{sec:orthogonalrep} and \ref{sec:manifold}.\vskip0.3cm
To formulate the problem precisely, let \(G\) be a symmetry group inherent in a neural network and acting on its parameter space. The interplay among architectural structure, learning setting, data distribution, activation functions, and loss specification can endow the resulting loss function $\mathscr L$ with $G$-invariance, i.e. $\mathscr L(g\cdot \theta)=\mathscr L(\theta),g\in G.$
When the group action is isometric, this invariance implies the \(G\)-equivariance of its gradient:
\begin{equation}\label{eq: equivaraince}
\nabla \mathscr L(g\cdot \theta)=g\cdot \nabla \mathscr L(\theta),\;\; g\in G.
\end{equation}
%In this work, we aim to provide a systematic theoretical characterization of the symmetries of critical points arising from this broad class of \(G\)-invariant loss functions. In particular, we seek to determine which symmetry types must occur and the corresponding topological weights are present, thereby providing a local structural description of the underlying loss landscape.\vskip 0.3cm
\noindent We therefore seek the zeros of \(\nabla\mathscr L\), which are precisely the critical points of \(\mathscr L\). The equivariance allows us to obtain integer-valued topological weights associated with orbit types. Let $\Omega$ be a prescribed region under certain conditions (see section \ref{sec: terminology} for more details), we define the symmetry index
\begin{equation}\label{eq:index}
\mathcal I_G(\mathscr L,\Omega)=\Big\{\big((K_{j_1}),n_{j_1}\big),\big((K_{j_2}),n_{j_2}\big),\cdots,\big((K_{j_s}),n_{j_s}\big)\Big\} 
\end{equation}
by extracting the \textbf{maximal orbit types} among the orbit types with nonzero coefficients associated with \(\nabla\mathscr L\) on \(\Omega\), written as
\begin{equation}\label{eq: extract}
\mathbf a=\sum_{j=1}^mn_j(K_j)\in U(G).
\end{equation}
Here, \(U(G)\) is the Euler ring (see appendix \ref{app:Euler} for definition and properties of Euler ring), and each integer coefficient \(n_j,\;j=1,\cdots,m\) is associated with the conjugacy class \((K_j),\) where $K_j$ is a subgroup of $G$. A nonzero coefficient certifies the existence of a critical orbit whose isotropy type is at least $(K_j),$ and the selected maximal orbit types are denoted by \((K_{j_1}),\ldots,(K_{j_s})\). The underlying degree theories and computational methods are introduced in Sections \ref{sec:orthogonalrep} and \ref{sec:manifold}.  One is also referred to \cite{balanov2025degree,balanov2006applied,balanov2010short,erbe1992s1,dabkowski2017twisted,gkeba2008some,gkeba1997degree,gkeba1994equivariant,gkeba2005classification,rybicki2005degree} and the references therein for more comprehensive details about the constructions of equivariant degrees, as well as their applications \cite{garcia2021global, balanov2021patterns, berezovik2017dihedral,balanov2018equivariant,garcia2019solutions,liu2025leaky}. %\textcolor{red}{Moreover, \(N_K\) gives the exact number of critical \((K)\)-orbits. Comparing \(N_K\) with \(n_K\), together with the tools from Morse indices, further distinguishes the nature of the critical points as local minima or saddle points.\vskip0.3cm
\begin{theorem}[Topological symmetry index of a neural-network] Suppose that the parameter space $V$ is a finite-dimensional orthogonal representation of a compact Lie group $G$, and that 
the loss function \(\mathscr L\) is twice continuously
differentiable and $G$-invariant. Let \(\Omega\subset V\) be an open bounded and \(G\)-invariant parameter subset such that \(\nabla\mathscr L\) is nonzero on \(\partial\Omega\). Then the symmetry index $\mathcal I_G(\mathscr{L},\Omega)$ defined in \eqref{eq:index} is well defined. Moreover, for every 
\[
\bigl((K_{j_r}),n_{j_r}\bigr)\in\mathcal I_G(\mathscr L,\Omega),\;\; r=1,\cdots,s,
\]
there exists \(\theta_o\in\Omega\) such that
\[
\nabla\mathscr L(\theta_o)=0,
\]
and its symmetry is at least $(K_{j_r}).$%Moreover,
%the number $N_H$ of critical $G$-orbits in $\Omega$ of orbit type
%$(H)$ satisfies
%\[
%N_H\ge |n_H|.
%\]
\end{theorem}
\noindent The same conclusion holds when $V$ is replaced by a $G$-manifold equipped with a $G$-invariant Riemannian metric and $\Omega$ is relatively compact (see Section \ref{sec:manifold}).\vskip0.3cm
The subsequent sections are arranged as follows. For the readers convenience, Section \ref{sec: terminology} collects the notations and definitions used throughout the paper. Given that the loss function in \eqref{eq: equivaraince} is of gradient type, Sections \ref{sec:orthogonalrep} and \ref{sec:manifold} present two gradient-related Equivariant Degree methods, together with the algorithms to determine \((K_j)\) and the corresponding \(n_{K_j},\; j=1,\cdots m\) in \eqref{eq: extract}. More precisely, in sections \ref{sec:recursive}-\ref{sec:slice}, one can find construction of gradient degree based on recursive formula and the method of reducing to slice. In section \ref{sec:isomorphism}, we focus on the linear isomorphism in which the ambient space \(V\) is an orthogonal \(G\)-representation, and the degree is constructed through the basic degree induced by irreducible decomposition of $V.$ On the other hand, Section \ref{sec:manifold} treats the general setting of a \(G\)-manifold \(V\). Without an underlying linear structure, the construction proceeds instead through orbit-type stratification and an inductive summation over the resulting strata. Section \ref{sec:example} then constructs several concrete toy examples to illustrate the application of the symmetry index: a softplus teacher–student model (Section \ref{sec:teacherstudent}), an $O(3)$-equivariant network (Section \ref{sec:steerable-slice}), and a transformer architecture (Section \ref{subsec:fmtr-recursive}), using the linear isomorphism method, the slice principle, and the recursive formula, respectively. Appendix \ref{app:Brouwer} recalls the construction and computational algorithm of the Equivariant Brouwer Degree, since, as the paper unfolds, one may quickly observe that the computation of the gradient-related Equivariant Degree methods ultimately can reduce to their Brouwer counterpart. Moreover, Appendices \ref{app:Euler} and \ref{app:aml} %, and \ref{app:computations} 
 present the properties of the Euler and Burnside rings, and the definition of amalgamated groups, respectively. We also collect some key calculations for Section \ref{sec:example} in Appendix \ref{app:computations}.
\vskip0.6cm
\section{Some Equivariant Terminology}\label{sec: terminology}
We first emphasize that the degree theories considered in this work require \(G\) to be a compact Lie group. Accordingly, the present analysis does not cover noncompact groups such as the positive scaling symmetries associated with ReLU activation.\vs
\noindent To construct the symmetry index $\mathcal I_G(\mathscr L,\Omega)$, one must first specify the following several ingredients:

\begin{itemize}
    \item The loss function $\mathscr{L}$ should be $G$-invariant, with $G$-equivariant gradient $\nabla\mathscr{L}$.

    \item The parameter space $V$ and the $G$ action on $V$ must be identified.%: depending on the architecture and parametrization, $V$ is either a linear space or a manifold, which determines whether the approach of Section~\ref{sec:orthogonalrep} or Section~\ref{sec:manifold} applies to constructing $\mathcal{I}_G(\mathscr{L},\Omega)$.

    \item A suitable domain $\Omega\subset V$ must be specified:
    $\Omega$ is required to be an open bounded $G$-invariant subset of $V$. Most importantly, $\Omega$ must be chosen so that
    \[
        \nabla \mathscr L(\theta)\neq 0,
        \;\;\theta\in\partial\Omega,
    \]
a condition on which the equivariant degree theory, and hence our approach, rests.
\end{itemize}
\noindent These constitute the basic ingredients needed for the construction of
$\mathcal I_G(\mathscr L,\Omega)$. Each arises naturally from the construction of equivariant degree, and its significance will become apparent as the
framework is developed below. The resulting $G$-equivariant gradient degree is an element of the Euler ring $U(G)$ of the form
\[
\nabla_G\text{-}\deg(\nabla\mathscr L,\Omega) = n_{1}(K_1)+n_{2}(K_2)+\cdots+n_{m}(K_m),
\]
from which the symmetry index $\mathcal{I}_G(\mathscr{L},\Omega)$ is read off directly. The precise definition of the degree, the term \((K_j)\) and \(n_{j}\), and the techniques appropriate to the different settings to derive this degree will be developed in detail in Sections~\ref{sec:orthogonalrep} and~\ref{sec:manifold}.\vs 
We next introduce the notation and preliminary
concepts that will be used throughout the work:
Unless
otherwise stated, every subgroup \(K\leq G\) is assumed to be closed. Let \(V\) be a \(G\)-space and let \(\theta\in V\). The following notations will be used throughout:

\begin{table}[ht]
\centering
\renewcommand{\arraystretch}{1.35}
\begin{tabular}{p{0.5\textwidth} p{0.50\textwidth}}
\hline
\textbf{Notation} & \textbf{Description} \\
\hline

$N(K)$
& Normalizer of $K$ in $G$ \\

$W(K)=N(K)/K$
& Weyl group of $K$ \\

$(K)=\{gKg^{-1}:g\in G\}$
& Conjugacy class of $K$ in $G$ \\

$\Phi(G)=\{(K):K\leq G\}$
& Collection of subgroups' conjugacy classes of $G$ \\

\small$\Phi_0(G)=\{(K)\in \Phi(G): W(K) \text{ finite}\}$
& Elements $(K)\in\Phi(G)$ for which $\dim W(K)=0$ \\

$G_{\theta}=\{g\in G:g {\theta}={\theta}\}$
& Isotropy group of $\theta\in V$ \\

$G(\theta)=\{g\theta:g\in G\}$
& $G$-orbit of $\theta$ \\

$(G_{\theta})$
& Orbit type of $\theta$ \\

\footnotesize $\Phi(G,V)=\{(K)\in \Phi(G): (K)=(G_{\theta}) \text{ for some }\theta\in V\}$
& Set of orbit types in $V$, i.e.,
  $\{(G_{\theta}):\theta\in V\}$ \\

$\Phi_0(G,V)=\Phi(G,V)\cap\Phi_0(G)$
& Orbit types in $V$ with finite Weyl group \\

$V^K=\{\theta\in V:K\leq G_{\theta}\}$
& $K$-fixed-point subspace of $V$ \\

$V_K=\{\theta\in V:G_{\theta}=K\}$
& Points with isotropy group exactly $K$ \\

$V_{(K)}=G(V_K)$
& Union of \(G\)-orbits of points whose isotropy is $K$ \\

\hline\caption{Notations}
\label{tab:notation}
\end{tabular}
\end{table}
\noindent An orbit whose isotropy group belongs to the conjugacy class $(K)$ is called a \emph{$(K)$-orbit}. The set $\Phi(G)$ is partially ordered by
\[
(K)\leq(H)
\quad\Longleftrightarrow\quad
gKg^{-1}\subseteq H
\text{ for some }g\in G.
\]
For classes in \(\Phi_0(G)\), let \(n(K,H)\) count the conjugates of \(H\) that contain \(K\), i.e.
\begin{equation}\label{eqn:nHK}
n(K,H)
:=
\left|\left\{H'\in(H):K\subseteq H'\right\}\right|.
\end{equation}
Notice that \(W(K)\) acts naturally on \(V^K\), and its action on \(V_K\) is free. Finally, a smooth manifold equipped with a smooth \(G\)-action will be called a \(G\)-manifold.\vskip0.3cm
We next collect several notions common to the degree constructions in the subsequent sections. For uniformity, we formulate them for a smooth $n$-dimensional \(G\)-manifold, an $n$-dimensional orthogonal \(G\)-representation constitutes a special case.\vskip 0.3cm
\noindent Equip the smooth \(n\)-dimensional \(G\)-manifold \(V\) with a \(G\)-invariant Riemannian metric, and choose a relatively compact invariant open subset \(\Omega\). We write \(C_G^2(V,\mathbb R)\) for the twice continuously differentiable $G$-invariant functions on \(V\). For \(\varphi\in C_G^2(V,\mathbb R)\), its gradient is an equivariant section of the tangent bundle:
\[
\nabla\varphi:V\longrightarrow TV,
\qquad
\nabla\varphi(\theta)\in T_\theta V.
\]
\begin{definition}[\(G\)-Equivariant Gradient Admissible Pair]
We call the pair \((\nabla\varphi,\Omega)\) a
\(G\)-equivariant gradient admissible pair if $\nabla\varphi(\theta)\neq 0,\; \theta\in \partial \Omega,$ i.e. $\nabla\varphi$ doesn't vanish on the boundary of $\Omega.$ 
\end{definition}\vskip0.3cm
 
\noindent Denote by \(\mathcal M_G^\nabla(V,TV)\) all the admissible pairs in $V$. Allowing \(V\) to range over finite-dimensional smooth \(G\)-manifolds, we define
\[
\mathcal M_G^\nabla
:=
\bigcup_V\mathcal M_G^\nabla(V,TV).
\]\vskip 0.3cm
\begin{definition}[$G$-Gradient $\Omega$-Admissible Homotopy] We call a homotopy $\nabla\Phi:[0,1]\times V\longrightarrow TV $ \(G\)-gradient $\Omega$-admissible if, for every \(t\in[0,1]\),
\[
\nabla\Phi_t:=\nabla\Phi(t,\cdot):V\longrightarrow TV
\]
is a \(G\)-equivariant gradient vector field and $\nabla \Phi(\theta) \neq 0,\theta\in \partial \Omega.$
\end{definition}
We now introduce the two algebraic structures in which the equivariant
degrees considered below take their values.
\begin{definition}[Euler ring]
The \emph{Euler ring} of $G$ is the free $\mathbb{Z}$-module
\[
  U(G) := \bigoplus_{(K)\in\Phi(G)} \mathbb{Z}(K),
\]
so that each $a \in U(G)$ has a unique representation
\[
  a = \sum_{(K)\in\Phi(G)} n_K (K), \qquad n_K \in \mathbb{Z},
\]
with finitely many $n_K$ nonzero. Multiplication on $U(G)$ is defined on generators by
\[
  (K) * (H) := \sum_{(S)\in\Phi(G)} c_S (S), \qquad (K),(H) \in \Phi(G),
\]
where
\begin{equation}\label{eq:cS}
  c_S := \chi_c\bigl((G/K \times G/H)_S / N(S)\bigr),
\end{equation}
and extended bilinearly to all of $U(G)$. Here $\chi_c$ denotes the Euler characteristic with respect to Alexander--Spanier cohomology with compact support.
\end{definition}

\begin{definition}[Burnside ring]
The \emph{Burnside ring} of $G$ is the free $\mathbb{Z}$-module generated by the classes with finite Weyl group,
\[
  A(G) := \bigoplus_{(K)\in\Phi_0(G)} \mathbb{Z}(K),
\]
whose elements are written as 
\[
b = \sum_{(K)\in\Phi_0(G)} n_K (K)
\]
with $n_K \in \mathbb{Z}$. Multiplication on $A(G)$ is defined by
\[
  (K) \cdot (H) := \sum_{(S)\in\Phi_0(G)} n_S (S), \qquad (K),(H) \in \Phi_0(G),
\]
where $n_S$ is the number of $(S)$-orbits in $G/K \times G/H$, and extended bilinearly to all of $A(G)$.
\end{definition}
\noindent Notice that \(n_S\) can be determined recursively from
\begin{equation}
n_S
=
\frac{
n(S,K)|W(K)|\,n(S,H)|W(H)|
-
\displaystyle\sum_{(\widetilde S)>(S)}
n_{\widetilde S}\,
n(S,\widetilde S)\,
|W(\widetilde S)|
}{
|W(S)|
}.
\end{equation}
\vskip.3cm 
The two rings are related by the ring homomorphism $\pi_0 \colon U(G) \to A(G),$
\begin{equation}\label{eq:eulerhomomorphism}
\pi_0((K))
=
\begin{cases}
(K), & (K)\in\Phi_0(G),\\
0,   & (K)\notin\Phi_0(G).
\end{cases}
\end{equation}
As will be seen in the subsequent sections, the homomorphism \(\pi_0\)
provides the key reduction from the Euler-ring-valued equivariant gradient
degree to its Burnside-ring-valued counterpart, and will play an essential
role in the computation of the corresponding basic degrees. For more details regarding Burnside ring and Euler ring, including some of their properties and corresponding proofs, one is referred to Appendix \ref{app:Euler}.\vskip0.6cm
\section{Equivariant Gradient Degree on Orthogonal $G$-representations}\label{sec:orthogonalrep} 
Having established the common notations and algebraic framework, we now introduce the equivariant gradient degree, beginning with orthogonal \(G\)-representations. For a complete construction of the  gradient degree, together with its properties and full proofs, we refer the readers to \cite{ gkeba1997degree, dabkowski2017twisted}.\vskip0.3cm
\subsection{Methodology based on recursive construction.}\label{sec:recursive}
Assume that \(V\) is an \(n\)-dimensional orthogonal representation and that \((\nabla\varphi,\Omega)\) satisfies the gradient-admissibility conditions defined in Section \ref{sec: terminology}. To explain the construction of the $G$-equivariant gradient degree and its properties, we first introduce the class of special $\Omega$-Morse functions. These functions also provide good representatives of $\nabla\varphi$ up to $\Omega$-admissible $G$-gradient homotopy, in parallel with regular normal maps for the equivariant Brouwer degree in Appendix \ref{app:Brouwer}. To define them, we begin with the normal bundle $\pi:\mathcal N^{(K)}\to V_{(K)}$ of an orbit-type stratum $V_{(K)},$ whose fiber over $\theta \in V_{(K)}$ is
\[
\mathcal N_{\theta}^{(K)}
=
(T_{\theta}V_{(K)})^\perp\subset T_{\theta}V.
\]
Endowed with the inner product inherited from $V$, $\mathcal N^{(K)}$ is a Euclidean $G$-bundle over $V_{(K)}$. For an invariant open subset $U_0\subset V_{(K)}$ and $\varepsilon>0$, the normal disk bundle of radius $\varepsilon$ over $U_0$ is
\[
\mathcal N(U_0,\varepsilon)
=
\{y\in\mathcal N^{(K)}:\pi(y)\in U_0,\ \|y\|<\varepsilon\}.
\]
On a suitable invariant neighborhood \(W\) of the zero section, the exponential map embeds the normal bundle equivariantly into \(V\).
\begin{definition}[$(K)$-Normal Set]
An invariant open set \(U\subset V\) is \((K)\)-normal when it is the exponential image of a normal disk bundle over an invariant open part of \(V_{(K)}\): $U=\exp\bigl(\mathcal N(U_0,\varepsilon)\bigr),$
where $\mathcal N(U_0,\varepsilon)\subset W.$

\end{definition}
\begin{definition}[$(K)$-Normal Vector Field]
Let \(\nabla\varphi:V\to TV\) be an \(\Omega\)-admissible vector field. Then \(\nabla\varphi\) is called (\(K)\)-normal if there exist an \((K)\)-normal subset $U=\exp\bigl(\mathcal{N}(U_0,\varepsilon)\bigr)$
and a smooth invariant function $\varphi_0:U_0\to\mathbb{R}$
such that
\begin{enumerate}
    \item[a)] $\nabla\varphi^{-1}(0)\cap V_{(K)}\cap\Omega\subset U_0,$
    \item[b)] $U\subset\Omega,$
    \item[c)] for all \(\theta\in U\), $\varphi(\theta)
    =
    \varphi_0(\pi(y))
    +\frac{1}{2}\|y\|^2,$
    where $\theta=\exp(y),
    \quad
    y\in\mathcal{N}(U_0,\varepsilon).$
\end{enumerate}
\end{definition}
\begin{definition}[$\Omega$-Morse Function]
A \(G\)-gradient admissible map \(\nabla\varphi\) is special \(\Omega\)-Morse provided that
\begin{enumerate}
    \item[a)] $\nabla\varphi|_{\Omega}$ is of class $C^1$.
    \item[b)] All its zero orbits in \(\Omega\) are regular.
    \item[c)] $\nabla\varphi$ is $(K)$-normal for every orbit type $(K)$ occurring in its zero set in $\Omega$.
\end{enumerate}
\end{definition}
\noindent The existence of such representatives, together with the $\Omega$-admissibility of the corresponding homotopy, follows from the results in \cite{mayer1989g}. We next define the \textit{$G$-Equivariant Gradient Degree}.
\begin{definition}[$G$-Equivariant Gradient Degree]\label{def:equigra}
The \(G\)-equivariant gradient degree assigns to each pair $(\nabla\varphi,\Omega) \in \mathcal{M}^{\nabla}_G$ an element of the Euler ring $U(G)$. We denote this degree by \(\nabla_G\text{-}\deg(\nabla\varphi,\Omega)\) and write it as
\begin{equation}\label{eqn: degaa}
\nabla_G\text{-deg}(\nabla\varphi,\Omega)=\sum_{j=1}^mn_{j}(K_j)=n_{1}(K_1)+\cdots+n_{m}(K_m),\;\;\;\; (K_j)\in \Phi(G)
\end{equation}
and we call it the \textit{$G$-equivariant gradient degree} of $\nabla\varphi$
in $\Omega$.
\end{definition}
\begin{theorem}
\label{thm:Ggrad-properties} The degree
$\nabla_{G}\mbox{\rm -}\deg:\mathcal M_G^{\nabla}\to U(G)$ introduced in Definition \ref{def:equigra} has the following properties:
\begin{enumerate}
    
\item[\textnormal{($\mathcal{G}_1$)}] \textbf{Additivity.}
If two disjoint invariant open subsets \(\Omega_1,\Omega_2\subset\Omega\) together contain all zeros in \(\Omega\), the degree splits into their contributions:
\[
\nabla_G\text{-deg}(\nabla\varphi,\Omega)
=
\nabla_G\text{-deg}(\nabla\varphi,\Omega_1)
+
\nabla_G\text{-deg}(\nabla\varphi,\Omega_2).
\]

\item[\textnormal{($\mathcal{G}_2$)}] \textbf{Homotopy.}
Along any \(\Omega\)-admissible \(G\)-gradient homotopy $\nabla\Phi_t:V\longrightarrow V,$ where $t\in [0,1],$
the degree is preserved:
\[
\nabla_G\text{-deg}(\nabla\Phi_t,\Omega)=\nabla_G\text{-deg}(\nabla\Phi_0,\Omega),\qquad t\in [0,1].
\]

\item[\textnormal{($\mathcal{G}_3$)}] \textbf{Normalization.}
For a special \(\Omega\)-Morse function whose critical set in \(\Omega\) consists of \(G(\theta_o)\), and set \((K)=(G_{\theta_o})\), then
\[
\nabla_G\text{-}\deg(\nabla\varphi,\Omega)
=
(-1)^{m^-(\nabla^2\varphi(\theta_o))}(K).
\]
Here \(m^-(\nabla^2\varphi(\theta_o))\) is the dimension of the negative eigensubspace of the Hessian $\nabla^2\varphi(\theta_o)$.
\item[\textnormal{($\mathcal{G}_4$)}] \textrm{\textbf{Existence.}} If $n_j\neq 0$ is a coefficient of $\nabla_G\text{-}\deg(\nabla\varphi,\Omega)=\sum_{j=1}^m n_j(K_j)$, then $\nabla\varphi$ possesses a zero $\theta_o\in\Omega$ with $(G_{\theta_o})\ge (K_j)$.

\item[\textnormal{($\mathcal{G}_5$)}] \textrm{\textbf{Product.}} The degree of the product of two admissible gradient maps is obtained by multiplying their degrees in \(U(G)\):
\[
\nabla_G\text{-deg}
(\nabla\varphi_1\times\nabla\varphi_2,
 \Omega_1\times\Omega_2)
=
\nabla_G\text{-deg}(\nabla\varphi_1,\Omega_1)
*
\nabla_G\text{-deg}(\nabla\varphi_2,\Omega_2),
\]
where $*$ denotes multiplication in $U(G)$.
\item[\textnormal{($\mathcal{G}_6$)}] \textbf{Suspension.}
Let $W$ be an orthogonal $G$-representation and let $B(W)$
be its unit ball. Then
\[
\nabla_G\text{-deg}
(\nabla\varphi\times\operatorname{Id}_W,
 \Omega\times B(W))
=
\nabla_G\text{-deg}(\nabla\varphi,\Omega).
\]

\item[\textnormal{($\mathcal{G}_7$)}] \textbf{Hopf Property.}
Two admissible \(G\)-gradient maps on the unit ball \(B(V)\) have equal degrees precisely when they can be connected by a \(B(V)\)-admissible homotopy consisting of \(G\)-gradient maps.

\item[\textnormal{($\mathcal{G}_8$)}] \textbf{Functoriality.}
Let
$\psi\colon G_0 \hookrightarrow G$ be an embedding of Lie groups. Then $\psi$
induces a $G_0$-action on $V$ under which $\nabla\varphi$ is an $\Omega$-admissible
$G_0$-gradient map, and
\[ \Psi\big[\nabla_G\text{-}\deg(\nabla\varphi,\Omega)\big] = \nabla_{G_0}\text{-}\deg(\nabla\varphi,\Omega),
\]
where $\Psi\colon U(G) \to U(G_0)$ denotes the homomorphism of Euler rings
induced by $\psi$.

\item[\textnormal{($\mathcal{G}_9$)}] \textbf{Reduction.}
The homomorphism $\pi_0$ from \eqref{eq:eulerhomomorphism} maps the gradient degree to the equivariant Brouwer degree:
\[
\pi_0\!\left[
\nabla_G\text{-deg}(\nabla\varphi,\Omega)
\right]
=
G\text{-deg}(\nabla\varphi,\Omega).
\]
\end{enumerate}
\end{theorem}
For proofs of some of the properties stated above, we refer the reader to Section 6 of \cite{DK}.\vs
\begin{remark}
It is important to note that the computation of the $G$-equivariant gradient degree $\nabla_G\text{-deg}(\nabla\varphi,\Omega)$ can be reduced, via the truncation $\pi_0$, to that of the equivariant Brouwer degree. Let $\mathcal J_o:=\{\,j : \dim W(K_j)=0\,\}.$
 For $j\in \mathcal J_o$, the coefficients $n_j$ can be computed via the recursive formula,
\begin{equation}\label{eqn:recur}
n_j=\frac{\deg\bigl((\nabla\varphi)^{K_j},\Omega^{K_j}\bigr)-\sum_{i\in \mathcal J_o,\ (K_i)>(K_j)} n_i\, n(K_j,K_i)\,|W(K_i)|}{|W(K_j)|}.
\end{equation}
Here $\deg\bigl((\nabla\varphi)^{K_j},\Omega^{K_j}\bigr)$ denotes the local Brouwer degree (the non-equivariant counterpart). One is referred to \cite{balanov2025degree} Chapter 2 for the details of local Brouwer degree and its computation. Consequently,
\begin{equation}
\nabla_G\text{-deg}(\nabla\varphi,\Omega)=\sum_{j\in \mathcal J_o} n_j(K_j)+\boldsymbol{\alpha} ,
\end{equation}
where $\boldsymbol{\alpha}=\sum_{j\notin \mathcal J_o} n_j(K_j)$ collects all terms with $\dim W(K_j)>0$, which are annihilated by $\pi_0$.
\end{remark}

\vs
\subsection{Degree on the Slice.}\label{sec:slice} In many applications, one needs to compute the gradient degree on a neighborhood of an isolated orbit $G(\theta_o)$. In this case, the computation reduces to the normal slice at $\theta_o$, on which the relevant symmetry is the isotropy group $G_{\theta_o}$.\vskip0.3cm
\begin{theorem}[Slice Principle]\label{thm:slice}
Let $\theta_o\in V$. We denote by $V_{o}:=T_{\theta_o}G(\theta_o)$
the tangent space to the orbit $G(\theta_o)$ at $\theta_o$, and let $S_o:=V_o^\perp$
be the slice to $G(\theta_o)$ at $\theta_o$. For sufficiently small \(\varepsilon>0\), let
\[
B_\varepsilon(S_o):=\{v\in S_o:\|v\|<\varepsilon\},
\]
and $\mathcal U:=G\bigl(\theta_o+B_\varepsilon(S_o)\bigr)$
be a tubular neighborhood of \(G(\theta_o)\), and suppose that \((\nabla\varphi,\mathcal U)\in\mathcal M_G^\nabla\). Define \(\varphi_o:S_o\to\mathbb R\) by
\[
\varphi_o(v):=\varphi(\theta_o+v).
\]
Then
\[
\nabla_G\text{-}\deg(\nabla\varphi,\mathcal U)
=
\Psi\!\left(
\nabla_{G_{\theta_o}}\text{-}\deg
\bigl(\nabla\varphi_o,B_\varepsilon(S_o))\right),
\]
where $\Psi: U(G_{\theta_o})\to U(G)$ is the ring homomorphism defined on
generators $\Psi((H)_{G_{\theta_o}})=(H)_G$.
\end{theorem}

\subsection{\textbf{Methodology Through Symmetric Linear Isomorphisms and Basic Degree}}\label{sec:isomorphism}
As in the case of the equivariant Brouwer degree (see Appendix~\ref{app:Brouwer}), the equivariant gradient degree admits a linearization procedure that reduces its computation to basic degrees.

Let $L : V \to V$ be a $G$-equivariant, self-adjoint linear isomorphism. Such an operator arises, for example, as the Hessian of $\varphi\in C^2_G(V,\mathbb R)$ at a $G$-fixed nondegenerate critical point $\theta_o\in V^G$:
\[
L = D(\nabla\varphi)(\theta_o)=\nabla^2\varphi(\theta_o).
\]
%which is self-adjoint since $\varphi$ is $C^2$, and $G$-equivariant since $\nabla\varphi$ is $G$-equivariant and $\theta_o$ is $G$-fixed. 
Since $L$ is self-adjoint, it is the gradient of the $G$-invariant quadratic function
\[
\widetilde\varphi(\theta):=\tfrac12\langle L\theta,\theta\rangle, \qquad \nabla\widetilde\varphi(\theta)=L\theta,
\]
so that $\nabla_G\text{-deg}(L,B(V))$ is well defined. For sufficiently small \(\varepsilon>0\), the local degree at \(\theta_o\) satisfies
\[
\begin{aligned}
&\nabla_G\text{-}\deg
\bigl(\nabla\varphi,B_\varepsilon(\theta_o)\bigr)\\
&\quad=
\nabla_G\text{-}\deg
\bigl(\theta\mapsto L(\theta-\theta_o),B_\varepsilon(\theta_o)\bigr)\\
&\quad=
\nabla_G\text{-}\deg\bigl(L,B_\varepsilon(0)\bigr)
=
\nabla_G\text{-}\deg\bigl(L,B(V)\bigr).
\end{aligned}
\]
The first equality follows from local linearization, the second from the \(G\)-equivariant translation \(v=\theta-\theta_0\), and the last from excision. The translation is \(G\)-equivariant because \(\theta_0\in V^G\). Decompose $V$ into its isotypic components,
\[
V=\bigoplus_{i=0}^{r} V_i ,
\]
where \(V_i\) contains the copies of the irreducible representation \(\mathcal V_i\) (see Definition \ref{def:isotypic}). Write \(\delta(L)\) for the spectrum, \(\delta^-(L)\) its negative part, and \(E(\beta)\) the eigenspace of \(\beta\in \delta(L)\). An admissible equivariant gradient homotopy deforms \(L\) to \(-\mathrm{Id}\) on its negative eigenspace and to \(\mathrm{Id}\) on the positive one. For each negative eigenvalue, the number of copies of \(\mathcal V_i\) is given by
\[
\mu_i(\beta)
=
\frac{\dim(E(\beta)\cap V_i)}{\dim\mathcal V_i}.
\] 
Applying ($\mathcal G_2$) to this deformation and ($\mathcal G_5$) to the resulting factors yields 
\[ \nabla_G\text{-}\deg(L,B(V)) = \prod_{\beta\in\delta^-(L)} \prod_{i=0}^r \left[ \nabla_G\text{-}\deg \bigl(-\operatorname{Id},B(\mathcal V_i)\bigr) \right]^{\mu_i(\beta)}. \] 
\begin{definition}[Gradient $G$-equivariant Basic Degree] \label{de:basicequigradient}
The gradient basic degree associated with \(\mathcal V_i\) is the degree of the negative identity on its open unit ball:
\begin{equation}\label{eqn:basicdegree}
\nabla_G\text{-}\deg_{\mathcal V_i}
:=
\nabla_G\text{-}\deg\bigl(-\operatorname{Id},B(\mathcal V_i)\bigr)
\in U(G),
\end{equation}
and has expansion
\begin{equation}
\nabla_G\text{-}\deg_{\mathcal V_i}
=\sum_{j=1}^{m} n_j\,(K_j),
\qquad (K_j)\in\Phi(G)\text{ and } n_j\in\mathbb Z .
\end{equation}
with only finitely many nonzero coefficients.
\end{definition}\vs
With this notation, the degree of \(L\) is the product of the basic degrees, each raised to its negative multiplicity:
\begin{equation}
\nabla_G\text{-}\deg(L,B(V)) = \prod_{\beta\in\delta^-(L)} \prod_{i=0}^r \left[ \nabla_{G}\mbox{\rm -}\deg_{\mathcal V_i} \right]^{\mu_i(\beta)}.
\end{equation}
\begin{proposition}
\label{prop:invertible} The basic degrees $\nabla_{G}\text{-}\deg_{\mathcal{V}_{i}}$ defined in \eqref{eqn:basicdegree} are invertible elements in $U(G)$.
\end{proposition}\vs
\noindent Appendix \ref{app:Brouwer} develops the corresponding counterpart Brouwer basic degrees and presents their construction, computational formulas, and algebraic properties.
\vskip.6cm 
\section{Equivariant Gradient Degree on $G$-Manifolds}\label{sec:manifold}
Although the isotypic decomposition of \(V\) developed in Section \ref{sec:orthogonalrep} provides a powerful computational tool, it relies essentially on the linear structure of \(V\). When \(V\) is a more general manifold, this approach is no longer applicable, and a different framework is required. This is precisely the setting addressed in this section.
\subsection{$G$-Equivariant Gradient Degree on Manifolds}\label{sec:degree_manifold}Let \((\nabla\varphi,\Omega)\) be an admissible gradient pair on a smooth \(n\)-dimensional \(G\)-manifold \(V\), where \(\Omega\) is invariant, open, and relatively compact. As in Table \ref{tab:notation}, \(\Phi(G,\Omega)\) records the orbit types occurring in this region. 
\begin{theorem}
The $G$-equivariant (manifold) gradient degree of $\nabla\varphi$ on $\Omega$ takes the form
\begin{equation}
\label{eq:grad-deg_manifold}\nabla_{G}^M\mbox{\rm -}\deg(\nabla\varphi,\Omega)=
\sum_{(K_{j})\in\Phi(G,\Omega)}n_{{j}}(K_{j})= n_{{1}}(K_{1})+\dots
+n_{{m}}(K_{m}),
\end{equation}
where %\textcolor{red}{$\nabla_{G}^M\mbox{\rm -}\deg(\nabla\varphi,\Omega)\in \mathcal A(G),$ a free abelian group generated by $\Theta_G(\Omega),$} and  
$n_{j}=\chi_{K_j}^G(\nabla\varphi_{j},V_{(K_j)}\cap\Omega),$ and it
is the standard Euler characteristic , which is also called the intersection number in \cite{hirsch2012differential}. Moreover, the resulting degree has the following properties:
\begin{enumerate}
    \item \textbf{(Existence)}
    Whenever \(n_{j}\neq0\) in \eqref{eq:grad-deg_manifold}, a critical point \(\theta_o\in\Omega\) exists with \((G_{\theta_o})\geq(K_j)\).
    \item \textbf{(Additivity)}
    If all critical points in \(\Omega\) lie in the union of two disjoint invariant open subsets, their degree contributions add: 
\[
    \nabla^M_{G}\mbox{\rm -}\deg
(\nabla\varphi,\Omega)=\nabla^M_{G}\mbox{\rm -}\deg
(\nabla\varphi,\Omega_1)
+\nabla^M_{G}\mbox{\rm -}\deg
(\nabla\varphi,\Omega_2).
\]
    \item \textbf{(Homotopy)}
    Let $\nabla \Phi_t:\overline\Omega\to TV, \;t\in[0,1]$, be an $\Omega$-admissible homotopy, then
$\nabla^M_{G}\mbox{\rm -}\deg
(\nabla \Phi_t,\Omega)$
is independent of $t\in [0,1].$
    \item \textbf{(Excision)}
    Let $\widetilde\Omega\subset\Omega$ be an open $G$-invariant subset and $(\nabla\varphi)^{-1}(0)\cap\Omega
        \subset
        \widetilde\Omega .$
    Then
\[
\nabla^M_{G}\mbox{\rm -}\deg
(\nabla\varphi,\Omega)
=\nabla^M_{G}\mbox{\rm -}\deg
(\nabla\varphi,\widetilde\Omega).
\]
\end{enumerate}
\end{theorem}\vskip0.3cm
\subsubsection{\textbf{Computations of the Gradient \texorpdfstring{$G$}{G}-Equivariant Degree on $G$-Manifolds}} 
A general $G$-manifold $V$ does not admit the isotypic decomposition used in Section \ref{sec:orthogonalrep}. Instead, the degree is computed recursively over the orbit-type strata of $V$, starting from a maximal orbit type. The normal bundle $\mathcal N^{(K)}$, the normal disk bundles $\mathcal N(U_0,\varepsilon)$, and the notions of $(K)$-normal sets and $(K)$-normal vector fields from Section \ref{sec:recursive} extend directly to here. However, $(K)$-normality here is imposed on one stratum at a time, in order to isolate the contribution of each orbit type. We recall the localization results needed for this procedure (see \cite{gkeba1997degree} for details). For a closed subgroup $K\le G$, let $\Omega_{(K)}=V_{(K)}\cap\Omega$.
 
\begin{proposition}\label{prop: alpha_normal}
Choose invariant open sets \(U_0,U_1\subset V_{(K)}\) and \(\varepsilon>0\) satisfying
\[
(\nabla\varphi)^{-1}(0)\cap\Omega_{(K)}
\subset U_0\subset\overline{U_0}\subset U_1,
\]
with $\Omega_0=\exp\bigl(\mathcal N(U_1,2\varepsilon)\bigr)\subset\Omega.$
Set \(U=\exp(\mathcal N(U_0,\varepsilon))\), and choose a smooth invariant cutoff \(\eta:V\to[0,1]\) that equals one near \(U\) and has support contained in \(\Omega_0\). Define
\[
\widetilde\varphi(\exp\widetilde \theta)
=
\varphi(\pi(\widetilde \theta))
+\frac12\|\widetilde \theta\|^2,
\]
and use the cutoff to obtain
\[
\varphi_0(\theta)=
\begin{cases}
\eta(\theta)\widetilde\varphi(\theta)+(1-\eta(\theta))\varphi(\theta),
&\theta\in\Omega_0,\\
\varphi(\theta),&\theta\notin\Omega_0.
\end{cases}
\]
The resulting gradient is \((K)\)-normal and agrees with the original field on \(V_{(K)}\) and outside \(\Omega_0\).
\end{proposition}
\begin{proposition}
    Fix a maximal orbit type \((K)\in\Phi(G,\Omega)\). An admissible gradient homotopy \(\Phi\) admits a \((K)\)-normal modification \(\Phi'\) supported in an invariant open subset \(\Omega_0\subset\Omega\). For every \(t\in[0,1]\), this modification agrees with \(\Phi\) both on \(V_{(K)}\) and outside \(\Omega_0\):
\[
\Phi'(t,\theta)=\Phi(t,\theta)
\qquad
\bigl(\theta\in V_{(K)}\cup(V\setminus\Omega_0)\bigr).
\]
\end{proposition}\vskip0.3cm
Begin with a maximal orbit type \((K_0)\). Proposition \ref{prop: alpha_normal} supplies a normal replacement \(\nabla\varphi_0\), differing from \(\nabla\varphi\) only within a small invariant tubular neighborhood \(\Omega_0\) and agreeing with it on \(V_{(K_0)}\).  Moreover,
\[
(\nabla\varphi_0)^{-1}(0)\cap\Omega
    \subset
    \Omega_{(K_0)}\cup \Omega_0^{'}.
\]
Place the remaining zeros in an invariant open set \(\Omega'_0\) whose closure avoids \(V_{(K_0)}\). The tubular neighborhood can be chosen disjoint from this remainder. Homotopy invariance, additivity, and excision then separate the contribution of \((K_0)\) from the residual degree, giving
\begin{align}\label{eq:recursive_formula}
\nabla_G^M\text{-}\deg(\nabla\varphi,\Omega)
&=\nabla_G^M\text{-}\deg(\nabla\varphi_0,\Omega)\\\notag
&=\nabla_G^M\text{-}\deg(\nabla\varphi_0,\Omega_0)+\nabla_G^M\text{-}\deg(\nabla\varphi_0,\Omega_0^{'})\\\notag
&=\nabla_G^M\text{-}\deg(\nabla\varphi_0,\Omega_{(K_0)})+\nabla_G^M\text{-}\deg(\nabla\varphi_0,\Omega_0^{'})\\\notag
&=\nabla_G^M\text{-}\deg\Big((\nabla \varphi)|_{V_{(K_0)}},\Omega_{(K_0)}\Big)+\nabla_G^M\text{-}\deg(\nabla\varphi_0,\Omega_0^{'})\\\notag
 &= n_{0}(K_0)
+\nabla_G^M\text{-}\deg(\nabla\varphi_0,\Omega_0^{'}).\\\notag
\end{align}
We then repeat the same construction on the remaining invariant set $\Omega_0^{'}.$ More precisely, for the next orbit type $(K_1),$ we choose an invariant neighborhood $\Omega_1$ of $\Omega_{(K_1)}$ inside $\Omega_0^{'}$ and construct an $(K_1)$-normal replacement $\nabla \varphi_1$ of $\nabla \varphi_0.$ Notice that, by construction, the induced gradient vector field on $V_{
(K_1)}$ remains unchanged. Namely, the restricted gradients of $\varphi_1,\varphi_0$ and $\varphi$ coincide on $V_{(K_1)}.$ Applying the same homotopy, additivity, and excision arguments gives a second contribution $n_{1},$ together with a new remainder. Iterating this procedure over all orbit types in $\Phi(G,\Omega)$, we obtain
\[
\nabla_{G}^M\mbox{\rm -}\deg(\nabla\varphi,\Omega)=
\sum_{(K_j)\in\Phi(G,\Omega)}n_{{j}}(K_{j}),
\]
where \((K_0),\ldots,(K_m)\) are ordered by successively selecting a maximal orbit type among those remaining. Here $n_{j}
    =
    \chi^G_{K_j}
    \bigl(\nabla\varphi_{j},\Omega_{(K_j)}\bigr)$ denotes the standard Euler-characteristic associated with the orbit type $(K_j).$
If the critical points of \(\varphi_j|_{V_{(K_j)}}\) in \(\Omega_{(K_j)}\) are
isolated and nondegenerate, then this contribution is given by the Morse-index formula
\[
n_{j}
=
\sum_{\theta_o\in
\Omega_{(K_j)}}
(-1)^{m^-(\theta_o)}.
\]
Here \(m^-(\theta_o)\) denotes the dimension of the negative eigenspace of the Hessian of \(\varphi_j|_{V_{(K_j)}}\) at \(\theta_o\).\vskip0.6cm

\section{Examples}\label{sec:example}
We now compute the symmetry index for three concrete neural-network models.
\subsection{A Softplus Teacher--Student MLP with Quartic Regularization}\label{sec:teacherstudent} In this section, we consider a two-layer teacher--student model with softplus
activation, fixed unit output weights, and Gaussian inputs
\(x\sim\mathcal N(0,I_k)\). The student network has $k$ weight vectors
\(\theta_1,\ldots,\theta_k\in\mathbb R^k\), and the teacher weights are the
standard basis vectors \(e_1,\ldots,e_k\). Writing \(\theta\in V:=\mathbb R^{k\times k}\) for the matrix with rows
\(\theta_i^T\), and \(\|\cdot\|\) for the Frobenius norm, the loss function is given by
\begin{equation}
\label{eq:softplus-loss}
\mathscr{L}(\theta)
=
\frac12\,
\mathbb{E}_{x\sim\mathcal{N}(0,I_k)}
\left[\left(
\sum_{i=1}^k \sigma(\theta_i^T x)
-
\sum_{j=1}^k \sigma(e_j^T x)\right)^2
\right]
+
\frac{\lambda}{4}\|\theta\|^4,
\end{equation}
where \(\sigma\) denotes the softplus activation
\[
\sigma(z)=\frac{1}{\beta}\log\bigl(1+e^{\beta z}\bigr),
\]
and \(\lambda>0\) is the regularization parameter. Let
\[
s(z):=\sigma'(z)=\frac{1}{1+e^{-\beta z}},
\qquad
r(x;\theta):=\sum_{i=1}^{k}\sigma(\theta_i^{T}x)-\sum_{j=1}^{k}\sigma(e_j^{T}x).
\]
A direct computation shows that the gradient
\(\nabla_{\theta}\mathscr{L}:V\to V\) has the form
\[
\nabla_{\theta}\mathscr{L}(\theta)=
\begin{bmatrix}
\nabla_{\theta_1}\mathscr{L}(\theta),\,
\nabla_{\theta_2}\mathscr{L}(\theta),\,
\ldots,\,
\nabla_{\theta_k}\mathscr{L}(\theta)
\end{bmatrix}^T,
\]
where, for \(i=1,\ldots,k\),
\begin{equation} \label{eq:lg}
\nabla_{\theta_i}\mathscr{L}(\theta)
=
\mathbb E_{x\sim\mathcal N(0,I_k)}
\bigl[r(x;\theta)\,s(\theta_i^{T}x)\,x\bigr]
+\lambda\|\theta\|^{2}\theta_i .
\end{equation}
Consider the action of \(G=S_k\times S_k\) on \(V\) given by
\[
(\tau,\gamma)(\theta_1,\ldots,\theta_k)^T
=
\bigl(\gamma\theta_{\tau^{-1}(1)},\ldots,\gamma\theta_{\tau^{-1}(k)}\bigr)^T,
\qquad (\tau,\gamma)\in G.
\]
Clearly, \(\nabla_\theta\mathscr{L}\) is \(G\)-equivariant, 
and the parameters fixed by \(G\) are precisely the constant matrices:
\[
V^G=\{cJ:c\in\mathbb{R}\},
\qquad
J=\mathbf{1}\mathbf{1}^T.
\]
\vskip0.3cm
\noindent Our goal is to compute the symmetry index associated with the critical points of \(\mathscr L\), which satisfy
\begin{equation}\label{eqn:critic}
\nabla \mathscr{L}(\theta)=0.
\end{equation}
To this end, we apply the linearization method of  \ref{sec:isomorphism}. 
First, an a priori estimate places all solutions of \eqref{eqn:critic} inside \(B_{\mathcal R}(0)\), where an admissible equivariant gradient homotopy connects \(\nabla\mathscr{L}\) to the identity map. We next identify a \(G\)-fixed nondegenerate critical point \(\theta_o\) and verify the invertibility of the \(G\)-equivariant Hessian
\[
L:=\nabla^2\mathscr L(\theta_o).
\]
In the translated coordinates \(v=\theta-\theta_o\), the gradient $v\longmapsto\nabla\mathscr L(\theta_o+v)$
is admissibly \(G\)-gradient homotopic to \(v\mapsto Lv\) on \(B_\varepsilon(0)\) for sufficiently small \(\varepsilon>0\). Choosing \(\varepsilon\) so that \(\overline{B_\varepsilon(\theta_o)}\subset B_R(0)\), we define
\[
\Omega:=B_\mathcal R(0)\setminus\overline{B_\varepsilon(\theta_o)}.
\]
Additivity and homotopy then give
\[
\begin{aligned}
\nabla_G\text{-}\deg(\nabla\mathscr{L},\Omega)
&=\nabla_G\text{-}\deg(\mathrm{Id},B_\mathcal R(0))
-
\nabla_G\text{-}\deg(\theta\mapsto L(\theta-\theta_o),B_\varepsilon(\theta_o))\\
&=\nabla_G\text{-}\deg(\mathrm{Id},B_\mathcal R(0))
-
\nabla_G\text{-}\deg(L,B_\varepsilon(0)).
\end{aligned}
\]\vskip0.3cm
For illustrative purposes, we consider the following particular setting where $k=4,\;\beta=1,\;\lambda=0.01.$ We emphasize, however, that the same approach extends to other widths $k$ and hyperparameters $\beta$ and $\lambda$ whenever the required admissibility and nondegeneracy conditions hold. By direct computation, one derives the critical point
\[
\theta_o=c_oJ\in V^G,
\]
where $c_o=0.26352$ and an \textit{a priori bound} given by $\mathcal R=10.$ (See appendix \ref{app:computations_1} for more details of the computations.)
%Appendix \ref{app:computations} derives the bound \(\mathcal R\) and establishes the admissible homotopy to the identity on \(B_{\mathcal R}(0)\). 
We now evaluate the Hessian at the invariant critical point
\[
\theta_o=0.26352J,
\quad
J=\mathbf{1}\mathbf{1}^{T}.
\]
Notice that at $\theta_o=c_oJ$, all preactivations coincide:
\[
(\theta_o)_i^Tx=c_o\mathcal S,
\;\; \text{ where }
\mathcal S:=\sum_{j=1}^{4}x_j,\;\; i=1,\ldots,4,
\]
and the residual becomes $r(x;\theta_o)
=
4\sigma(c_o\mathcal S)-\sum_{j=1}^{4}\sigma(x_j).$
Differentiating \eqref{eq:lg} at $\theta_o$, we obtain
$L=[L_{ij}]_{i,j=1}^{4},$
where
\[
\begin{aligned}
L_{ij}:=\nabla^2_{\theta_i\theta_j}\mathscr L(\theta_o)=&\mathbb E_{x\sim \mathcal N(0,I_k)}\!\left[s(c_o\mathcal S)^2\,xx^T\right]\\
&+\delta_{ij}\mathbb E_{x\sim \mathcal N(0,I_k)}\!\left[
r(x;\theta_o)s(c_o\mathcal S)\bigl(1-s(c_o\mathcal S)\bigr)\,xx^T
\right]\\
&+0.16c_o^2\delta_{ij}I_4+0.02c_o^2J,
\qquad \qquad i,j=1,\ldots,4.
\end{aligned}
\]
Here $\delta_{ij}$ denotes the Kronecker delta. Setting
\[
\begin{aligned}
\mathcal C&:=\mathbb E_{x\sim \mathcal N(0,I_k)}\!\left[s(c_o\mathcal S)^2\,xx^T\right],\\
\mathcal D&:=\mathcal C+\mathbb E_{x\sim \mathcal N(0,I_k)}\!\left[
r(x;\theta_o)s(c_o\mathcal S)\bigl(1-s(c_o\mathcal S)\bigr)\,xx^T
\right],
\end{aligned}
\]
the diagonal and off-diagonal blocks of $L$ are, respectively,
\[
\widetilde {\mathcal D}
=
\mathcal D+0.16c_o^2I_4+0.02c_o^2J,
\qquad
\widetilde{\mathcal C}
=
\mathcal C+0.02c_o^2J.
\]
Thus
\begin{equation}
\label{eq:softplus-hessian-blocks}
L=
\begin{pmatrix}
\widetilde{\mathcal D} & \widetilde{\mathcal C} & \widetilde{\mathcal C} & \widetilde{\mathcal C}\\
\widetilde{\mathcal C} & \widetilde{\mathcal D} & \widetilde{\mathcal C} & \widetilde{\mathcal C}\\
\widetilde{\mathcal C} & \widetilde{\mathcal C}& \widetilde{\mathcal D} & \widetilde{\mathcal C}\\
\widetilde{\mathcal C} & \widetilde{\mathcal C} & \widetilde{\mathcal C} & \widetilde{\mathcal D}
\end{pmatrix}
=
I_4\otimes(\widetilde{\mathcal D}-\widetilde{\mathcal C})
+J\otimes\widetilde{\mathcal C}.
\end{equation} We analyze its spectrum using the representation decomposition from Section 2 of \cite{liu2025leaky}. Let
\[
V_0=\operatorname{span}\{\mathbf1\},
\quad
V_\perp=\{x\in\mathbb R^4:\mathbf1^Tx=0\},
\]
where $\mathbf1=(1,1,1,1)^T.$
Then
$\mathbb R^4\cong V_0\oplus V_\perp$ and
\begin{align}\label{eq:V_dcom}
V:=\mathbb R^{4^2}&\cong({\mathbb R^{4}})^{\otimes 2}=(V_0\otimes V_0)\oplus (V_0\otimes V_{\perp})
\oplus (V_{\perp}\otimes V_0)\oplus (V_{\perp}\otimes V_{\perp}).
\end{align}
 Table \ref{tbl:characters_s4} records the characters \(\chi_j\) of the five irreducibles \(\mathcal V_j\).
 \begin{table}[htbp]
    \centering
    \caption{Character table of $S_4$.}
    \label{tbl:characters_s4}
    \begin{tabular}{|c|c|c|c|c|c|c|}
        \hline
        Rep. & Character
        & $(1)$ & $(1,2)$ & $(1,2)(3,4)$
        & $(1,2,3)$ & $(1,2,3,4)$ \\
        \hline
        $\mathcal{V}_0$ & $\chi_0$
        & $1$ & $1$ & $1$ & $1$ & $1$ \\
        $\mathcal{V}_1$ & $\chi_1$
        & $1$ & $-1$ & $1$ & $1$ & $-1$ \\
        $\mathcal{V}_2$ & $\chi_2$
        & $2$ & $0$ & $2$ & $-1$ & $0$ \\
        $\mathcal{V}_3$ & $\chi_3$
        & $3$ & $-1$ & $-1$ & $0$ & $1$ \\
        $\mathcal{V}_4$ & $\chi_4$
        & $3$ & $1$ & $-1$ & $0$ & $-1$ \\
        \hline
        $\mathbb{R}^4$ & $\chi_{\mathbb{R}^4}$
        & $4$ & $2$ & $0$ & $1$ & $0$ \\
        \hline
    \end{tabular}
\end{table}
 Since characters multiply under outer tensor products,
$\chi_{\mathcal V_i\otimes\mathcal V_j}(g,h)
=
\chi_i(g)\chi_j(h),$ therefore,
the decomposition of the parameter representation can be written as follows:

\[
V=\mathbb R^{16}=(\mathcal V_{0}\otimes \mathcal V_0)\oplus (\mathcal V_{0}\otimes \mathcal V_4)\oplus (\mathcal V_{4}\otimes \mathcal V_0)\oplus (\mathcal V_{4}\otimes \mathcal V_4).%\textcolor{red}{=\mathcal V_1\oplus \mathcal V_{12}\oplus \mathcal V_{15}\oplus \mathcal V_{25}.}
\] 
%\[
%\nabla\text{-deg}_{\mathcal W_{j}\otimes {\mathcal W_{j}}}=\nabla\text{-deg}_{\mathcal W_{j}}\cdot \nabla\text{-deg}_{\mathcal W_{j}}
%\]
Notice $\mathcal V_{0}$ can be identified with the trivial representation and $\mathcal V_{4}$ is the standard representation. %Similarly, $\mathcal W_{5}\cong S^{(3,2)}$ and $\mathcal W_{4}\cong S^{(3,1,1)}.$ 
By applying formula \eqref{eq:softplus-hessian-blocks}, one has the spectrum of the Hessian
\begin{equation}
\label{eq:spectrum}
\delta(L)=\begin{cases}
1.1472\qquad  \text{
with mult= }1\\
0.9486\qquad  \text{ with mult }3\\
-0.0452\quad \text{ with mult=} 3\\
-0.1128\quad \text{ with mult=} 9\\
\end{cases}
\end{equation}
and \[
E^-(L)=
(\mathcal V_4\otimes\mathcal V_0)
\oplus
(\mathcal V_4\otimes\mathcal V_4).
\]
Since $L=\nabla^2_{\theta} \mathscr{L}(\theta_o)$ is invertible, there exists $\varepsilon>0$
sufficiently small such that $\theta_o$ is the only zero of $\nabla \mathscr{L}$ in
$\overline{B_\varepsilon(\theta_o)}$ and
$\nabla\mathscr{L}$
is admissibly $G$-homotopic on $B_\varepsilon(\theta_o)$ to its linearization.
Then
\begin{align*}
\label{eq:annulus-degree}
\nabla_G\text{-}\deg(\nabla_{\theta}\mathscr{L},\Omega)
&=
\nabla_G\text{-}\deg(\nabla_{\theta}\mathscr{L},B_{\mathcal R}(0))
-
\nabla_G\text{-}\deg(\nabla_{\theta}\mathscr{L},B_\varepsilon(\theta_o))\\
&=\nabla_G\text{-}\deg(\operatorname{Id},B_{\mathcal R}(0))
-
\nabla_G\text{-}\deg(L,B_\varepsilon(0))\\
&=(G)- \nabla_G\text{-}\deg_{\mathcal V_4\otimes \mathcal V_0}*\nabla_G\text{-}\deg_{\mathcal V_4\otimes \mathcal V_4}.
\end{align*}
For the finite group \(G=S_4\times S_4\), we have \(\Phi(G)=\Phi_0(G)\), and the two basic degrees are given by
\begin{align*}
\nabla_G\text{-}\deg_{\mathcal V_{4}\otimes \mathcal V_0}&=(S_4\times S_4)-(\mathbb Z_1\times S_4)+3(D_1\times S_4)-\textcolor{red}{(D_2\times S_4)}-\textcolor{red}{2(D_3\times S_4)}\\
\nabla_G\text{-}\deg_{\mathcal V_{4}\otimes \mathcal V_4}&=(S_4\times S_4)+(\mathbb Z_1\times \mathbb Z_1)-(\mathbb Z_2^{\mathbb Z_1}\times^{\mathbb Z_1}\mathbb Z_2)+2(\mathbb Z_1\times D_1)+2(D_1\times\mathbb Z_1)-3(D_1^{\mathbb Z_1}\times ^{\mathbb Z_1}D_1)-(\mathbb Z_3^{\mathbb Z_1}\times ^{\mathbb Z_1}\mathbb Z_3)\\
&-(D_2^{\mathbb Z_1}\times ^{\mathbb Z_1} D_2)+(D_2^{\mathbb Z_1}\times ^{\mathbb Z_1}V_4)-(V_4^{\mathbb Z_1}\times ^{\mathbb Z_1}V_4)-(D_2\times \mathbb Z_1)+(V_4^{\mathbb Z_1}\times^{\mathbb Z_1}D_2)+(D_2^{D_1}\times ^{\mathbb Z_1}D_1)\\
&-5(D_1\times D_1)-(D_1^{\mathbb Z_1}\times ^{D_1}D_2)-(\mathbb Z_1\times D_2)-(D_3\times \mathbb Z_1)-(\mathbb Z_1\times D_3)+4(D_3^{\mathbb Z_1}\times ^{\mathbb Z_1}D_3)\\&+(D_2^{D_1}\times ^{D_1}D_2)-\textcolor{red}{(D_4^{\mathbb Z_1}\times ^{\mathbb Z_1}D_4)}+(D_4^{D_2}\times ^{\mathbb Z_1}\mathbb Z_2)+2(D_1\times D_2)+2(D_2\times D_1)+(\mathbb Z_2^{\mathbb Z_1}\times ^{D_2}D_4)\\&+3(D_4^{\mathbb Z_1}\times ^{\mathbb Z_1}D_4)+3(D_1\times D_3)+3(D_3\times D_1)-\textcolor{red}{(D_2^{D_1}\times ^{D_2}D_4)}-\textcolor{red}{(D_4^{D_2}\times ^{D_1}D_2)}-\textcolor{red}{2(S_4^{\mathbb Z_1}\times ^{\mathbb Z_1}S_4)}\\
&-\textcolor{red}{(D_2\times D_3)}-\textcolor{red}{(D_3\times D_2)}-\textcolor{red}{(D_4^{D_2}\times ^{D_2}D_4)}-\textcolor{red}{2(D_3\times D_3)}.
\end{align*}

%where $\boldsymbol{\alpha_1},\boldsymbol{\alpha_2}$ denote the submaximal in the Euler ring.
\noindent Therefore,
\begin{align*}
\nabla_G\text{-}\deg(\nabla_{\theta}\mathscr{L},\Omega)
&=-(\mathbb Z_1\times \mathbb Z_1)+(\mathbb Z_2^{\mathbb Z_1}\times^{\mathbb Z_1}\mathbb Z_2)+2(\mathbb Z_1\times D_1)-(\mathbb Z_2^{\mathbb Z_1}\times ^{\mathbb Z_1}D_1)+2(D_1\times\mathbb Z_1)-(D_2^{\mathbb Z_1}\times ^{\mathbb Z_1}\mathbb Z_2)\\&+(D_1^{\mathbb Z_1}\times ^{\mathbb Z_1}D_1)
-(\mathbb Z_3^{\mathbb Z_1}\times ^{\mathbb Z_1}\mathbb Z_3)
+(V_4^{\mathbb Z_1}\times ^{\mathbb Z_1}V_4)-(V_4^{\mathbb Z_1}\times^{\mathbb Z_1}D_2)-(\mathbb Z_2^{\mathbb Z_1}\times ^{D_1}D_2)-(D_2^{D1}\times ^{\mathbb Z_1}D_1)\\
&-5(D_1\times D_1)-(\mathbb Z_1\times D_2)-(D_3\times \mathbb Z_1)-(\mathbb Z_1\times D_3)+\textcolor{red}{(D_4^{\mathbb Z_1}\times ^{\mathbb Z_1}D_4)}\\
&-(D_4^{D_2}\times ^{\mathbb Z_1}\mathbb Z_2)+(D_2^{\mathbb Z_1}\times ^{D_2}D_4)+2(D_1\times D_2)+(D_2\times D_1)-3(D_4^{\mathbb Z_1}\times ^{\mathbb Z_1}D_4)\\
&+3(D_1\times D_3)+3(D_3\times D_1)-(D_2\times D_2)-(D_2^{D_1}\times ^{D_2}D_4)+\textcolor{red}{(D_4^{D_2}\times ^{D_1}D_2)}+\textcolor{red}{2(S_4^{\mathbb Z_1}\times ^{\mathbb Z_1}S_4)}\\
&+(\mathbb Z_1\times S_4)-(D_2\times D_3)-(D_3\times D_2)+\textcolor{red}{(D_4^{D_2}\times^{D_2}D_4)}-2(D_3\times D_3)-3(D_1\times S_4)\\
&+\textcolor{red}{(D_2\times S_4)}+\textcolor{red}{2(D_3\times S_4)}.
\end{align*}
The symmetry index for $(\mathscr{L},\Omega)$ is therefore given by
\begin{equation}\label{ind:softplus}
\begin{aligned}
\mathcal{I}_G(\mathscr{L},\Omega)
=
\Big\{
&\bigl((D_2\times S_4),1\bigr),
\bigl((D_3\times S_4),2\bigr),\bigl((D_4^{\mathbb Z_1}\times ^{\mathbb Z_1} D_4),1\bigr),\bigl((D_4^{D_2}\times ^{D_2} D_4),1\bigr),\\
&\bigl((D_4^{D_2}\times ^{D_1} D_2),1\bigr), \bigl((S_4^{\mathbb Z_1}\times ^{\mathbb Z_1} S_4),2\bigr)
\Big\}.
\end{aligned}
\end{equation}
Figure~\ref{fig:placeholder} illustrates the patterns of the critical matrices associated with the six orbit types detected in the symmetry index \eqref{ind:softplus}, each of which substantially reduces the number of free parameters.
\vskip0.6cm
\begin{figure}[H]
    \centering
    \includegraphics[width=0.95\linewidth]{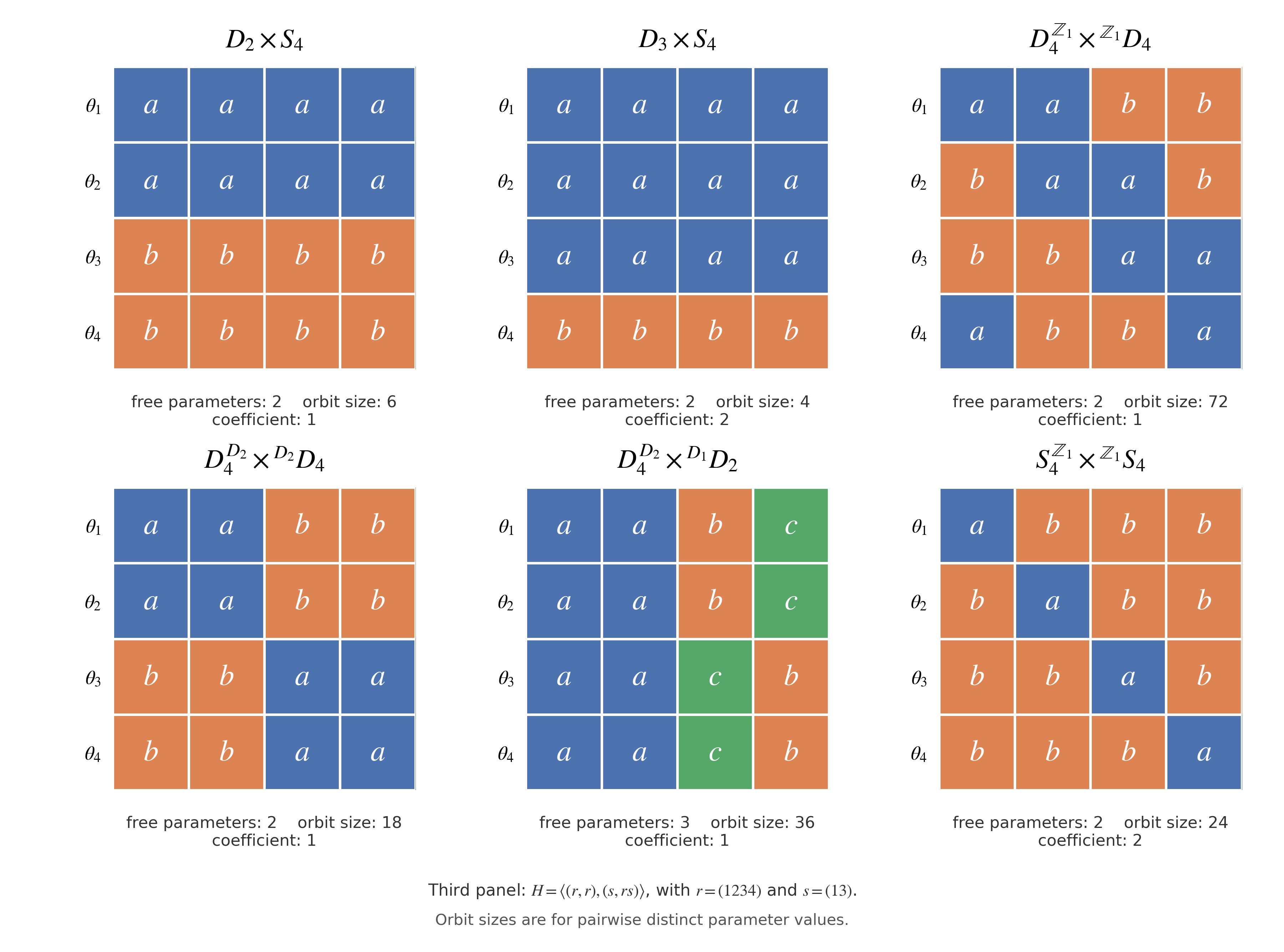}
    \caption{Patterns of Critical Matrix detected by $I_{G}(\mathscr L,\Omega)$}
    \label{fig:placeholder}
\end{figure}\vskip0.6cm
\subsection{A Steerable ENN with
\texorpdfstring{$O(3)\times S_4$}{O(3) x S4} Parameter Symmetry}
\label{sec:steerable-slice}
We construct a toy $O(3)$ steerable network with a smooth invariant regularizer.
Let \(I:\mathbb R^3\to\mathbb R\) be an integrable scalar input field,
with the \(O(3)\)-action
\[
(\omega\cdot I)(x):=I(\omega^{-1}x),
\qquad \omega\in O(3).
\]
Define the kernel \(\kappa(x):=e^{-\|x\|^2}x\) and the associated convolution $z_I:\mathbb R^3\to \mathbb R^3$ by
\[
z_I(x):=(\kappa*I)(x)
=\int_{\mathbb R^3}\kappa(x-y)I(y)\,dy.
\]
Since \(\kappa(\omega x)=\omega\kappa(x)\), the change of variables \(y=\omega u\) gives
\[
\begin{aligned}
\bigl(\kappa*(\omega\cdot I)\bigr)(x)
&=\int_{\mathbb R^3}\kappa(x-y)I(\omega^{-1}y)\,dy=\int_{\mathbb R^3}\kappa(x-\omega u)I(u)\,du\\
&=\omega\int_{\mathbb R^3}\kappa(\omega^{-1}x-u)I(u)\,du\\
&=\omega z_I(\omega^{-1}x)=(\omega\cdot z_I)(x),
\end{aligned}
\]
where \((\omega\cdot z_I)(x):=\omega z_I(\omega^{-1}x)\) is the standard action on vector fields. Thus, \(I\mapsto z_I\) is an
\(O(3)\)-equivariant linear operator.\vskip0.3cm 
Let $V=\R^{3\times4}\oplus\R^{3\times4}$ be the parameter space, equipped with the product Frobenius inner product, and write $\theta=(W,U)\in V$. For
\(\alpha=1,2,3\) and \(i=1,\ldots,4\), define
\[
h_{\alpha i}^{(1)}[I](x):=W_{\alpha i}z_I(x),
\;\;
h_{\alpha i}^{(2)}[I](x):=U_{\alpha i}z_I(x),
\]
and the output matrix field
\[
F_\theta[I](x)=
\begin{pmatrix}
F_\theta^{(1,1)}[I](x) & F_\theta^{(1,2)}[I](x)\\
F_\theta^{(2,1)}[I](x) & F_\theta^{(2,2)}[I](x)
\end{pmatrix}
\in\mathbb R^{8\times8},
\]
where, for \(r,s\in\{1,2\}\),
\begin{equation}\label{eq:enn-output-blocks}
\bigl(F_\theta^{(r,s)}[I](x)\bigr)_{ij}
=
\sum_{\alpha=1}^{3}
h_{\alpha i}^{(r)}[I](x)\cdot h_{\alpha j}^{(s)}[I](x).
\end{equation}
A direct computation gives
\[
F_\theta[I](x)=\|z_I(x)\|^2
\begin{pmatrix}
W^TW & W^TU\\
U^TW & U^TU
\end{pmatrix}.
\]
It follows that
\[
F_\theta[\omega\cdot I](x)=F_\theta[I](\omega^{-1}x),
\qquad \omega\in O(3).
\]
i.e. for each fixed \(\theta\), the network is spatially
\(O(3)\)-equivariant. We now choose input $I_0(x)=e^{-\|x\|^2}$ and the target field
\[
 Y(x)=\|z_{I_0}(x)\|^2
\begin{pmatrix}\mathcal C&0\\0&\mathcal D\end{pmatrix},
\]
where 
\begin{align*}
 \mathcal C=I_4-\mathbf{a}\mathbf{a}^{\mathsf T},\quad
 \mathcal D&=2\mathcal C+3\mathbf{a}\mathbf{a}^{\mathsf T}=2I_4+\mathbf{aa}^{\mathsf T}
 \end{align*}
 with $ \mathbf{a}=\tfrac12(1,1,1,1)^{\mathsf T}.$ Notice that
\[
 z_{I_0}(x)=\frac12\left(\frac\pi2\right)^{3/2}
                 x e^{-\|{x}\|^2/2},
\]
so $c:=\int_{\R^3}\|z_{I_0}(x)\|^4\,dx>0.$  The loss function $\mathscr{L}:V\to\R$ is therefore given by
\begin{align}
\mathscr L(W,U)
&=\frac{1}{4c}\int_{\mathbb R^3}
  \|F_\theta[I_0](x)-Y(x)\|_F^2\,dx \notag\\
&=\frac14\|W^TW-\mathcal C\|_F^2
 +\frac12\|W^TU\|_F^2
 +\frac14\|U^TU-\mathcal D\|_F^2.
\label{eq:enn-base-loss}
\end{align}
We refer to \(\mathscr L\) as the base loss.
Let \(P_\gamma\) denote the permutation matrix
associated with \(\gamma\in S_4\). The
group $G = O(3) \times S_4$ acts orthogonally on the parameter space by
\[
(\tau,\gamma)\cdot(W,U)
=(\tau WP_\gamma^T,\tau UP_\gamma^T).
\]
Since \(\mathcal C\) and \( \mathcal D\) commute with every \(P_\gamma\),
the base loss is \(G\)-invariant. We consider the regularized loss
\begin{equation}
\label{eq:enn-regularized-loss}
\widetilde{\mathscr L}(W,U)
=
\mathscr L(W,U)+\eta(q(W,U))Q(W,U),
\end{equation}
where \(Q:V\to[0,\infty)\) is a \(G\)-invariant polynomial,
specified in \eqref{eq:enn-regularizer} below, and
\[
q(W,U):=\|W^TW-C\|_F^2+\|U\|_F^2.
\]
For \(b>0\), to be chosen sufficiently small below, take
\(\eta\in C^\infty([0,\infty);[0,1])\) such that
\[
\eta(s)=
\begin{cases}
    1,\quad \text{ for }0\leq s\leq b/4\\
    0,\quad \text{ for }s\geq b/2
\end{cases}
\]
This regularizer changes the local degree near \(q^{-1}(0)\) while
preserving the loss outside \(\{q<b/2\}\).
In the remainder of this section, we study critical points of
\(\widetilde{\mathscr L}\):\vskip0.3cm

The gradient of the base loss
$\nabla \mathscr{L}:V\to V$ takes the form
\begin{equation}
 \begin{aligned}
 \nabla_W \mathscr{L}(W,U)&=W(W^{\mathsf T}W-\mathcal C)+UU^{\mathsf T}W,\\
 \nabla_U\mathscr{L}(W,U)&=U(U^{\mathsf T}U-\mathcal D)+WW^{\mathsf T}U.
 \end{aligned}
 \label{eq:gradient}
\end{equation}
%For further details regarding the equivariant convolution and tensor-product building blocks, one is referred to 
%\cite{GeigerSmidt,e3nnTP,Villar}. 
Consider $\theta_o=(T,0),$ where
\begin{equation}
 T=\frac12\begin{pmatrix}
 1&1&-1&-1\\
 1&-1&1&-1\\
 1&-1&-1&1
 \end{pmatrix}.
 \label{eq:T}
\end{equation}
A direct calculation shows that \(\theta_o\) is a critical point of \(\mathscr{L}\). Moreover, the isotropy of $\theta_o$ is
\begin{equation}
G_{\theta_o}
   =\Delta S_4=\{(\rho(\gamma),\gamma):\gamma\in S_4\}\cong S_4,
 \label{eq:isotropy}
\end{equation}
where $
\rho(\gamma):=TP_\gamma T^{T}\in O(3).$
In this section, we apply the slice principle \ref{thm:slice}, developed in Section~\ref{sec:slice},
to reduce the computation of the local $G$-equivariant gradient degree of $\nabla\widetilde{\mathscr{L}}$
near the critical orbit $G(\theta_o)$ to the computation of $\Delta S_4$-equivariant degree on the slice. More precisely, the tangent space to the orbit at \(\theta_o\), denoted by \(T_{\theta_o}G(\theta_o)\), is generated by the infinitesimal \(O(3)\)-action, and
\[
T_{\theta_o}G(\theta_o)
=
\{(AT,0):A\in\mathfrak{so}(3)\}.
\]
Choosing the standard basis
\[
J_1=
\begin{pmatrix}
0&0&0\\
0&0&-1\\
0&1&0
\end{pmatrix},
\quad
J_2=
\begin{pmatrix}
0&0&1\\
0&0&0\\
-1&0&0
\end{pmatrix},
\quad
J_3=
\begin{pmatrix}
0&-1&0\\
1&0&0\\
0&0&0
\end{pmatrix}
\]
of \(\mathfrak{so}(3)\), we obtain a basis $\{(J_1T,0),(J_2T,0),(J_3T,0)\}$
for \(T_{\theta_o}G(\theta_o)\).
Consequently, the slice at \(\theta_o\) is given by
\[
S_{\theta_o} = \{(A,B)\in V : \langle A, J_iT\rangle_F = 0,\ i=1,2,3\} = \{(A,B)\in V : AT^\top = TA^\top\}.
\]
Under the identification \(\Delta S_4\cong S_4\), the character of \(V\) has the values displayed below. Comparing these values with Table \ref{tbl:characters_s4} gives
\begin{center}
\label{tbl:characters}
\begin{tabular}
[c]{|c|c|c|c|c|c|c|}\hline
Rep. & Character & (1) & (1,2) & (1,2)(3,4) & (1,2,3) & (1,2,3,4)\\\hline
$V$ & $\chi_V$ & $24$ & $4$ & $0$ & $0$ &
$0$\\\hline
\end{tabular}
\end{center}
and 
\[
V\cong
2\mathcal V_0\oplus2\mathcal V_2
\oplus2\mathcal V_3\oplus4\mathcal V_4.
\] 
Moreover, since
\[
\dim T_{\theta_o}G(\theta_o)=\dim G(\theta_o)=\dim G-\dim G_{\theta_o}=3
\]
and the tangent space is isomorphic to $\mathcal V_3,$ we conclude that slice $S_{\theta_o}$ admits the following isotypic decomposition
\[
S_{\theta_o}=V_0\oplus V^*_0\oplus V_2\oplus V^*_2\oplus V_3\oplus V_4\oplus V^*_4\oplus V^{**}_4\oplus V^{***}_4,\quad\quad  
\] 
where $V_j=\mathcal V_j, (j=0,2,3,4)$ and $V_i, V_i^*\; (i=0,2)$ are equivalent, $V_4, V_4^*, V_4^{**}, V_4^{***}$ are equivalent.
\vskip0.3cm
We next compute the Hessian $
\nabla^2\mathscr{L}(\theta_o):V\to V$ and determine its eigenspaces. More precisely, 
\begin{equation}
 \nabla^2\mathscr{L}(\theta_o)=
 \begin{pmatrix}
  \mathcal M&0\\
  0&-(I_4+\mathbf{aa}^{\mathsf T})\otimes I_3
 \end{pmatrix},
 \qquad
 \mathcal M_{ij}=\delta_{ij}I_3+t_jt_i^{\mathsf T},
 \label{eq:hessian-matrix}
\end{equation}
where $t_j\in\mathbb{R}^3$ denotes the $j$-th column of $T,$ so that each $\mathcal M_{ij}$ is a
$3\times3$ block. A direct calculation gives the eigenvalues of \(\nabla^2\mathscr{L}(\theta_o)\) together with their multiplicities and associated eigenspaces, as listed in Table \ref{tab:hessian-slice}. 
\begin{table}[ht]
\centering
\renewcommand{\arraystretch}{1.25}
\begin{tabular}{c l c}
\hline
Eigenvalue & Eigenspace & Dimension \\
\hline
$2$  & $\mathcal{V}_0$ & $1$ \\
$2$  & $\mathcal{V}_2$ & $2$ \\
$2$  & $\mathcal{V}_4$ & $3$ \\
$1$  & $\mathcal{V}_4^{*}$ & $3$ \\
$-1$ & $\mathcal{V}_0^{*}$ & $1$ \\
$-1$ & $\mathcal{V}_2^{*}$ & $2$ \\
$-1$ & $\mathcal{V}_3$ & $3$ \\
$-1$ & $\mathcal{V}_4^{**}$ & $3$ \\
$-2$ & $\mathcal{V}_4^{***}$ & $3$ \\
$0$  & $\mathcal{V}_3^{*}$ & $3$ \\
\hline
\end{tabular}
\caption{$\delta\big(\nabla^2\mathscr{L}(\theta_o)\big)$ and the
corresponding eigenspaces}
\label{tab:hessian-slice}
\end{table}\vskip0.3cm
Notice that the restricted Hessian
\( L\big|_{S_{\theta_o}}:=\nabla^2\mathscr L(\theta_o)\big|_{S_{\theta_o}}:S_{\theta_o}\to S_{\theta_o}\)
is invertible and
\[
E^{-}\bigl( L\big|_{S_{\theta_o}}\bigr)
\cong\mathcal V_0\oplus\mathcal V_2\oplus\mathcal V_3\oplus 2\mathcal V_4 .
\]
We now specify the polynomial \(Q\) in
\eqref{eq:enn-regularized-loss}.
On \(\mathbb R^{4\times4}\), let \(S_4\) act by
\(X\mapsto P_\gamma XP_\gamma^T\).
For \(j=2,3\), denote the corresponding isotypic orthogonal
projection by
\begin{equation}
\label{eq:enn-projectors}
\Pi_jX
=
\frac{d_j}{24}
\sum_{\gamma\in S_4}
\chi_j(\gamma)\,P_\gamma XP_\gamma^T,
\qquad
d_j=\dim\mathcal V_j.
\end{equation}
On the class representatives listed in Table \ref{tbl:characters_s4},
\[
\chi_2=(2,0,2,-1,0),
\quad
\chi_3=(3,-1,-1,0,1).
\]
Each of these two representations occurs once in
\(\mathbb R^{4\times4}\) under the conjugation action. Set
\begin{equation}
\label{eq:enn-regularizer}
Q(W,U)
:=
\|\Pi_2(W^TU)\|_F^2
+
\|\Pi_3(W^TU)\|_F^2.
\end{equation}
The projections commute with the conjugation action, and
\(W^TU\) transforms as \(P_\gamma W^TU P_\gamma^T\).
Hence \(Q\) is nonnegative and \(G\)-invariant.
Together with the invariance of \(q\), this proves that
\(\widetilde{\mathscr L}\) is a smooth \(G\)-invariant loss.
Moreover, the function \(q\) is proper, \(q^{-1}(0)=G(\theta_o)\), and
\(D^2q(\theta_o)\) is positive definite on \(S_{\theta_o}\).
For sufficiently small \(r>0\), the sublevel set
\[
\mathcal U_r
:=
\{\theta\in V:q(\theta)<r\}
\]
is therefore a bounded \(G\)-invariant tubular neighborhood of
\(G(\theta_o)\). Fix \(b>0\) sufficiently small that
\(\overline{\mathcal U_b}\) contains no critical points of
\(\mathscr L\) outside \(G(\theta_o)\), and use this value of \(b\)
in the cutoff in \eqref{eq:enn-regularized-loss}.\vskip.3cm

Since \(\nabla Q(\theta_o)=0\), \(\theta_o\) remains a critical point
of \(\widetilde{\mathscr L}\). The only nonzero eigenvalue of
\(\nabla^2Q(\theta_o)\) is \(2\), with eigenspace \(\mathcal V_2^*\oplus\mathcal V_3\),
so regularization changes the Hessian eigenvalue on this subspace
from \(-1\) to \(1\). Therefore, for the invertible regularized Hessian $\widetilde L_{S_{\theta_o}}
:=
\left.
\nabla^2\widetilde{\mathscr L}(\theta_o)
\right|_{S_{\theta_o}},$ one has
\begin{equation}
\label{eq:enn-modified-negative}
E^-(\widetilde L_{S_{\theta_o}})\cong
\mathcal V_0\oplus2\mathcal V_4.
\end{equation}
Choose \(0<a<b/4\) sufficiently small that
\(\overline{\mathcal U_a}\) contains no critical points
of \(\widetilde{\mathscr L}\) outside \(G(\theta_o)\), and put
\begin{equation}
\label{eq:enn-annular-domain}
\Omega
=
\mathcal U_b\setminus\overline{\mathcal U_a}
=
\{\theta\in V:a<q(\theta)<b\}.
\end{equation}
Thus \(\Omega\) is a \(G\)-invariant annular region and \(G(\theta_o)\cap\Omega=\varnothing\). Since \(\eta(q)=0\) near \(\partial\mathcal U_b\), the homotopy
\[
\nabla\bigl(\mathscr L+t\eta(q)Q\bigr),
\qquad t\in[0,1],
\]
is an admissible \(G\)-gradient homotopy on \(\mathcal U_b\).
By additivity, homotopy, and the slice principle
with local linearization, we obtain, for sufficiently small \(r_0>0\),
\begin{align}
\nabla_G\text{-}\deg
(\nabla\widetilde{\mathscr L},\Omega)
&=
\nabla_G\text{-}\deg
(\nabla\widetilde{\mathscr L},\mathcal U_b)
-
\nabla_G\text{-}\deg
(\nabla\widetilde{\mathscr L},\mathcal U_a)
\notag\\
&=
\nabla_G\text{-}\deg
(\nabla\mathscr L,\mathcal U_b)
-
\nabla_G\text{-}\deg
(\nabla\widetilde{\mathscr L},\mathcal U_a)
\notag\\
&=
\Psi\!\left(
\nabla_{\Delta S_4}\text{-}\deg
(L_{S_{\theta_o}},B_{r_0}(S_{\theta_0})\right)
-
\Psi\!\left(\nabla_{\Delta S_4}\text{-}\deg
(\widetilde L_{S_{\theta_o}},B_{r_0}(S_{\theta_0})\right),
\label{eq:enn-degree-difference}
\end{align}
where $\Psi: U(\Delta S_4)\to U(O(3)\times S_4)$ is the ring homomorphism induced by
generators $\Psi(K)=(K)$, $(K)\in\Phi(\Delta S_4).$ Notice that $ S_4$ is discrete, and hence $\Phi(\Delta S_4)=\Phi_0(\Delta S_4).$ It then follows from Lemma \ref{le:involutive} that 
\begin{align}
\nabla_{\Delta S_4}\text{-}\deg\bigl(L_{S_{\theta_o}},B_{r_0}(S_{\theta_o})\bigr)
&=\prod_{\beta\in\delta^-(\nabla^2\mathscr{L}(\theta_o))}
  \Big(\prod_{i=0}^4\nabla_{\Delta S_4}\text{-}\deg_{\mathcal V_i}\Big)\notag\\
&=\nabla_{\Delta S_4}\text{-}\deg_{\mathcal V_0^*}*\nabla_{\Delta S_4}\text{-}\deg_{\mathcal V_2^*}*\nabla_{\Delta S_4}\text{-}\deg_{\mathcal V_3}*\nabla_{\Delta S_4}\text{-}\deg_{\mathcal V_4^{**}}*\nabla_{\Delta S_4}\text{-}\deg_{\mathcal V_4^{***}}\notag\\
&=\nabla_{\Delta S_4}\text{-}\deg_{\mathcal V_0}*\nabla_{\Delta S_4}\text{-}\deg_{\mathcal V_2}*\nabla_{\Delta S_4}\text{-}\deg_{\mathcal V_3}*\bigl(\nabla_{\Delta S_4}\text{-}\deg_{\mathcal V_4}\bigr)^2\notag\\
&=\nabla_{\Delta S_4}\text{-}\deg_{\mathcal V_0}*\nabla_{\Delta S_4}\text{-}\deg_{\mathcal V_2}*\nabla_{\Delta S_4}\text{-}\deg_{\mathcal V_3},
\label{eqn:gradientdegree}\\
\intertext{and}
\nabla_{\Delta S_4}\text{-}\deg\bigl(\widetilde{L}_{S_{\theta_o}},B_{r_0}(S_{\theta_o})\bigr)
&=\prod_{\beta\in\delta^-(\nabla^2\widetilde{\mathscr{L}}(\theta_o))}
  \Big(\prod_{i=0}^4\nabla_{\Delta S_4}\text{-}\deg_{\mathcal V_i}\Big)\notag\\
&=\nabla_{\Delta S_4}\text{-}\deg_{\mathcal V_0^*}*\nabla_{\Delta S_4}\text{-}\deg_{\mathcal V_4^{**}}*\nabla_{\Delta S_4}\text{-}\deg_{\mathcal V_4^{***}}\notag\\
&=\nabla_{\Delta S_4}\text{-}\deg_{\mathcal V_0}*\bigl(\nabla_{\Delta S_4}\text{-}\deg_{\mathcal V_4}\bigr)^2\notag\\
&=\nabla_{\Delta S_4}\text{-}\deg_{\mathcal V_0}.
\label{eqn:gradientdegree_1}
\end{align}
where
\begin{align*}
\nabla_{\Delta  S_4}\text{-}\deg_{\mathcal V_0}&=-(\Delta S_4)\\
\nabla_{\Delta S_4}\text{-}\deg_{\mathcal V_2}&=(\Delta S_4)+(\Delta V_4)-\textcolor{red}{2(\Delta D_4)}\\
\nabla_{\Delta S_4}\text{-}\deg_{\mathcal V_3}&=(\Delta S_4)-\textcolor{red}{(\Delta D_1)}-\textcolor{red}{(\Delta \mathbb Z_3)}-\textcolor{red}{(\Delta \mathbb Z_4)}+(\Delta \mathbb Z_1).
\end{align*}
Notice, all the maximal orbit types are highlighted in red. Therefore, 
\[
\nabla_{G}\text{-}\deg
\bigl(\nabla \widetilde{\mathscr{L}},\Omega\bigr)=(\Delta\mathbb Z_2)-(\Delta D_1)+\textcolor{red}{(\Delta\mathbb Z_3)}-(\Delta V_4)-(\Delta\mathbb Z_4)+\textcolor{red}{2(\Delta D_4)},
\]
and the symmetry index for $(\widetilde{\mathscr{L}}, \Omega)$ is given by
\begin{equation}\label{eq:ENN_index}
\mathcal{I}_G(\widetilde{\mathscr{L}},\Omega)
=
\Big\{
\bigl((\Delta\mathbb Z_3),1\bigr),
\bigl((\Delta D_4),2\bigr)
\Big\}.
\end{equation}
Figure~\ref{fig:ENN} illustrates the patterns of the critical matrices $U,W$ associated with the orbit types in the symmetry index \eqref{eq:ENN_index}.
\begin{figure}[H]
    \centering
    \includegraphics[width=0.8\linewidth]{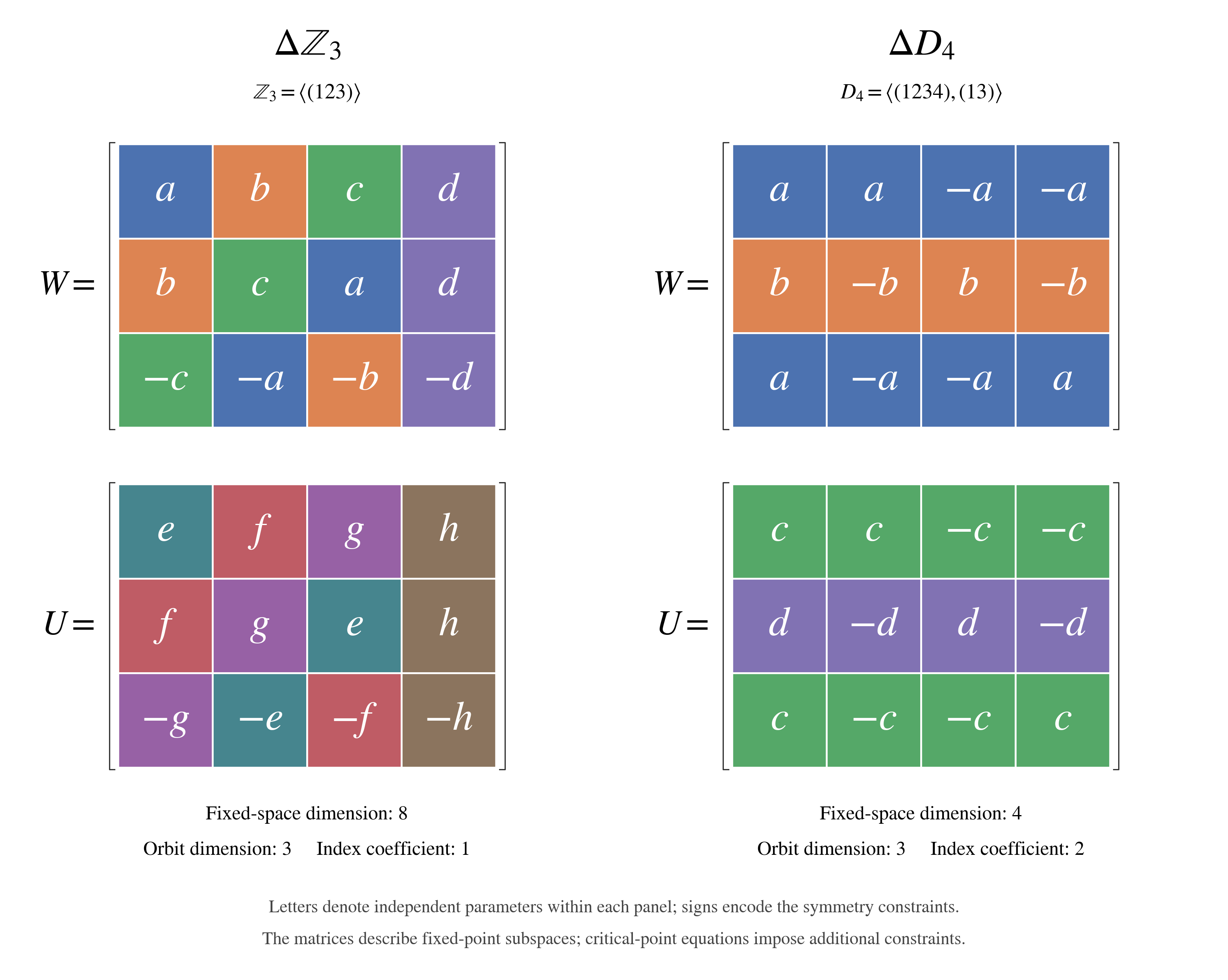}
    \caption{Patterns of Critical Matrices $U,W$ in $\mathcal I_{G}(\widetilde{\mathscr L},\Omega)$}
    \label{fig:ENN}
\end{figure}

\vskip0.3cm

\subsection{Transformer via the Recursive Formula}
\label{subsec:fmtr-recursive}
In this section, we construct a single-layer, single-head toy transformer with regularized loss to illustrate the recursive method in Section \ref{sec:recursive} for the study of \textbf{degenerate} critical points. Let $V=(\mathbb{R}^{3\times 3})^{\oplus 5}$ be equipped with the product Frobenius inner product. For
\[
\theta=(Q,K_e,B,O,M)\in V,
\]
where $Q$, $K_e$, $B$, $O\in \mathbb R^{3\times 3}$ are the query, key, value, and output projection matrices, and $M\in \mathbb R^{3\times 3}$ is the input weight of MLP. Set input sequence $X=(I_3,-I_3)^T\in\mathbb{R}^{6\times 3}$ and define the transformer architecture
\begin{equation}
\begin{aligned}
&S_\theta = \operatorname{softmax}\!\Big(\frac{XQK_e^TX^T}{\sqrt{3}}\Big)\\
&Y_\theta = X+S_\theta XBO\\
&F_\theta(X) = Y_\theta+\tanh(Y_\theta M), 
\end{aligned}
\end{equation}
with \(\operatorname{softmax}\) applied row by row and \(\operatorname{tanh}\) entrywise.
Let loss function be
\[
\mathscr{L}(\theta)= \tfrac12\,\|F_\theta(X)-X\|_F^2 .
\]
For $P\tau, P_\gamma$ the permutation matrices associated with $\tau,\gamma \in S_3,$ the group $G=S_3\times S_3$ acts orthogonally on $V$ by
\begin{equation}\label{eq:transformer_action}
(\tau,\gamma)\cdot\theta=(QP_\tau^T,\,K_eP_\tau^T,\,BP_\gamma^T,\,P_\gamma O,\,M),
\end{equation}
which preserves $QK_e^T$ and $BO$, and hence $F_\theta(X)$. Clearly, $\mathscr{L}$ is $G$-invariant.\vskip0.3cm

The loss $\mathscr{L}$ admits a family of trivial zero-loss critical points
\[
\left\{(Q, 0, B, 0, 0) : Q, B \in \mathbb R^{3\times 3}\right\}.
\]
Indeed, \(K_e=O=M=0\) implies \(Y_\theta=X\) and \(F_\theta(X)=X\). We introduce coordinates adapted to the $G$-action by separating the mean-zero parts of the first rows of $Q$ and $B$ from the remaining parameters. More precisely, let
\[
P=I_3-\frac13\mathbf{1}\mathbf{1}^T,
\;
E=\mathbf{1}^\perp\subset\mathbb R^3
\]
where $\mathbf{1}=(1,1,1)^T,$ so that \(P\) is the orthogonal projection onto the mean-zero subspace \(E\).
With \(e_1=(1,0,0)^T\), set
\[
v=PQ^Te_1,
\qquad
w=PB^Te_1,
\qquad
u=\theta-(e_1v^T,0,e_1w^T,0,0).
\]
Thus \(v,w\in E\) are the centered first rows of \(Q\) and \(B\),
respectively. The remaining component
\(u\) belongs to the orthogonal complement \(Z\), giving the orthogonal decomposition
\[
V\cong E\oplus E\oplus Z,
\]
where $\dim E=2,\;\dim Z=41.$ The first and second factors of \(G\) act on \(v\) and \(w\),
respectively, and
\(\|\theta\|^2=\|v\|^2+\|w\|^2+\|u\|^2\). Moreover, at $u = 0$, where $K_e = O = M = 0$, one has
\begin{equation}\label{eq:hahahah}
\begin{aligned}
&\mathscr L(v, w, 0) = 0, \\
&\nabla \mathscr L(v, w, 0) = 0. 
\end{aligned}
\end{equation}\vskip0.3cm
To define the regularizer, identify $E$ isometrically with \(\mathbb C\) through
$\zeta(v)
=
\sqrt{\frac32}\,v_1
+\frac{i}{\sqrt2}(v_2-v_3),$ and set
\[
\chi(t)=
\begin{cases}
0,
& t\le1/64,
\\[2pt]
\exp\!\left(-\dfrac1{t-1/64}\right),
& t>1/64,
\end{cases}
\]
and 
\[
q(v)=|\zeta(v)^3-1|^2.
\]
Notice \(S_3=\langle (123),(23)\rangle,\) and the generators act as multiplication
by \(e^{2\pi i/3}\) and complex conjugation, respectively. Thus $q$ is \(S_3\)-invariant, and the potential
\begin{equation}
\phi(v):=\int_0^{q(v)}\chi(t)\,dt.
\label{eq:phi}
\end{equation}
is smooth, nonnegative and $S_3$-invariant. Its flat region 
$\mathcal{A}=\{q\le 1/64\}$
consists of three disjoint closed topological disks, and $\phi$ is critical exactly on 
\[
\{0\}\cup\mathcal{A}.
\]
We refer the reader to
Appendix \ref{app:admissibility} for details of the auxiliary-potential construction and the characterization of its critical
set. \vskip0.3cm
We now use \(\phi\) to define a \(G\)-invariant regularized loss
and derive a priori estimates for its critical points to select
an admissible domain \(\Omega\). By the smoothness of \(\mathscr L\), there exists \(\mathcal C\ge0\) such that
\begin{equation}\label{eq:lowerbound}
\bigl\langle\nabla^2_{\theta}\mathscr L(\theta)\xi,\xi\bigr\rangle
\ge -\mathcal C\|\xi\|^2
\qquad
\text{for all }\|\theta\|\le 5/4,\quad \xi\in V.
\end{equation}
Choose \(\lambda>\mathcal C\) and define the regularized loss
\begin{equation}
\widetilde{\mathscr L}(\theta)
=
\mathscr L(\theta)
+\frac{\lambda}{2}\|u\|^2
+\phi(v)+\phi(w).
\label{eq:regularized-loss}
\end{equation}
Every nonzero critical point of \(\widetilde{\mathscr L}\) in
\(\overline{B_{5/4}(0)}\) satisfies the \textit{a priori} bounds
\[
\frac12
<
(7/8)^{1/3}
\le \|\theta\|
\le (9/8)^{1/3}
<
\frac54.
\]
We refer the reader to Appendix~\ref{app:admissibility}
for the proof of these estimates.
They lead us to choose the \(G\)-invariant annulus
\[
\Omega
=
B_{5/4}(0)\setminus\overline{B_{1/2}(0)},
\]
and
the pair \((\nabla\widetilde{\mathscr L},\Omega)\) is admissible. The critical family in this region is exactly
\begin{equation}
\bigl\{\theta\in\overline\Omega:
\nabla\widetilde{\mathscr L}(\theta)=0\bigr\}
=
(\mathcal A\times\{0\}\times\{0\})
\cup
(\{0\}\times\mathcal A\times\{0\})
\subset\Omega.
\label{eq:critical-family}
\end{equation}
Every critical point in this family has a four-dimensional
Hessian kernel, and no critical \(G\)-orbit is isolated.
The characterization in \eqref{eq:critical-family}
and these degeneracy properties are established
in Appendix~\ref{app:admissibility}.
\vskip0.3cm

We next apply the recursive formula~\eqref{eqn:recur} to compute the gradient degree $\nabla_{G}\text{-}\deg
\bigl(\nabla \widetilde{\mathscr{L}},\Omega\bigr)$. Write $\mathbb{Z}_3=\langle(123)\rangle$, $D_1=\langle(23)\rangle$ and $\mathbb{Z}_1=\{e\}$. By the $G$-action~\eqref{eq:transformer_action}, the stabilizers of the three column pairs of $(Q,K_e)$, and the three column--row pairs of $(B,O)$, are conjugate to $S_3$, $D_1$ or $\mathbb{Z}_1$. Note that $\mathbb{Z}_3$ does not occur as an isotropy, as it would degenerate to $S_3$ otherwise. Therefore, the orbit types in $V$ under $S_3\times S_3$ takes the form of direct products
\[
(K_1\times K_2),
\]
where $K_1,K_2\in\{S_3,\,D_1,\,\mathbb{Z}_1\}.$ By~\eqref{eq:critical-family}, every critical point of $\widetilde{\mathscr{L}}$ in $\Omega$ has one of the coordinates $v,w$ in $\mathcal{A}$ and the other at $0$. We therefore compute the local degrees of $\nabla\phi$ near $0$ and near $\mathcal{A}$. For $K\le S_3$, let $E^K$ be the $K$-fixed point subspace and define
\[
(\nabla\phi)^K:=\nabla\phi\big|_{E^K}\colon E^K\to E^K.
\]
Set
\[
b_K=\deg
\bigl((\nabla\phi)^K,\ B_{1/2}(0)\cap E^K\big),
\]
and, with $N^K=\{v\in E^K:\ q(v)<1/16\}$,
\[
a_K=\deg\big((\nabla\phi)^K,\ N^K\big).
\]
Both degrees are well defined, since $\nabla\phi$ has no zeros on $\partial B_{1/2}(0)$ or on $\{q=1/16\}$. Their values are given in Table \ref{tab:fixed-point-degrees}.
\begin{table}[h]
\centering
\begin{tabular}{cccc}
\hline
$K$ & $\dim E^K$ & $b_K$ & $a_K$\\
\hline
$S_3$ & $0$ & $1$ & $0$\\
$D_1$ & $1$ & $0$ & $1$\\
$\mathbb{Z}_1$ & $2$ & $-2$ & $3$\\
\hline
\end{tabular}
\caption{Fixed-point degrees of $\nabla\phi$ near $0$ and near $\mathcal{A}$.}
\label{tab:fixed-point-degrees}
\end{table}
Indeed, near $0$ the gradient is $-6\chi(1)\bar{z}^{2}$ to leading order ($b_{\mathbb Z_1}=-2$); on the reflection axis $E^{D_1}$ it is a negative multiple of $x^2$ ($b_{D_1}=0$). On each component of $N^{\mathbb{Z}_1}$, the gradient homotopy from $\phi$ to $q$ has no zeros on the boundary, and $q$ has a nondegenerate minimum there ($a_{\mathbb Z_1}=3$). Finally, each reflection axis meets exactly one such component ($a_{D_1}=1$).\vskip0.3cm
We now assemble these local degrees into the fixed-point degrees of $\nabla\widetilde{\mathscr{L}}$ on $\Omega$. For $K=K_1\times K_2\le G$, set
\[
d_K=\deg\big(\nabla\widetilde{\mathscr{L}}|_{V^K},\,\Omega^K\big),\qquad \Omega^K=\Omega\cap V^K .
\]
The two factors act separately on the two \(E\)-coordinates, while \(Z\) carries the induced \(G\)-action. Thus, one has $V^K=E^{K_1}\oplus E^{K_2}\oplus Z^K.$
Consider the homotopy
\[
\widetilde{\mathscr{L}}_t=t\,\mathscr{L}+\tfrac{\lambda}{2}\|u\|^2+\phi(v)+\phi(w),\qquad 0\le t\le 1,
\]
and the gradients \(\nabla\widetilde{\mathscr L}_t\) forms an admissible $G$-gradient homotopy (see Appendix \ref{app:admissibility}). At $t=0$,
\[
\nabla\widetilde{\mathscr{L}}_0(v,w,u)=\big(\nabla\phi(v),\,\nabla\phi(w),\,\lambda u\big).
\]
Fix $0<\rho<1/3$ and let
\[
\begin{aligned}
\mathcal U_1&=N\times B_{1/2}(0)\times B_\rho(0),\\
\mathcal U_2&=B_{1/2}(0)\times N\times B_\rho(0),
\end{aligned}
\]
where $N=\{v\in E: q(v)<1/16\}$. On $N$ we have $0.90<\|v\|<1.08$, so $\mathcal U_1,\mathcal U_2$ are disjoint subsets of $\Omega$ containing all zeros of $\nabla\widetilde{\mathscr{L}}_0$ in $\overline{\Omega}$. Writing $\mathcal U_i^K=\mathcal U_i\cap V^K$, and for $B_{1/2}(0)\subset E$, $B_\rho(0)\subset Z,$ by homotopy, additivity, product properties, we obtain
\begin{equation}\label{eq: dk}
\begin{aligned}
d_K
&=\deg\big(\nabla\widetilde{\mathscr{L}}_0|_{V^K},\,\Omega^K\big)\\
&=\deg\big(\nabla\widetilde{\mathscr{L}}_0|_{V^K},\,U_1^K\big)+\deg\big(\nabla\widetilde{\mathscr{L}}_0|_{V^K},\,\mathcal U_2^K\big)\\
&=a_{K_1}\,b_{K_2}\,\deg\big(\lambda\,\mathrm{Id},\,B_\rho(0)\cap Z^K\big)+b_{K_1}\,a_{K_2}\,\deg\big(\lambda\,\mathrm{Id},\,B_\rho(0)\cap Z^K\big)\\
&=a_{K_1}b_{K_2}+b_{K_1}a_{K_2}.
\end{aligned}
\end{equation}
Moreover, according to the conjugacy classes subgroup lattice of $S_3\times S_3$ (isotropies highlighted in red) in Figure \ref{fig:lattice}, one can derive the $|W_G(K)|$ and $n(K, H)$ required by recursive formula \eqref{eqn:recur}. See Tables \ref{tab:WeylS3} and \ref{tab:nHK} for details.
\begin{figure}
    \centering
\includegraphics[width=0.7\linewidth]{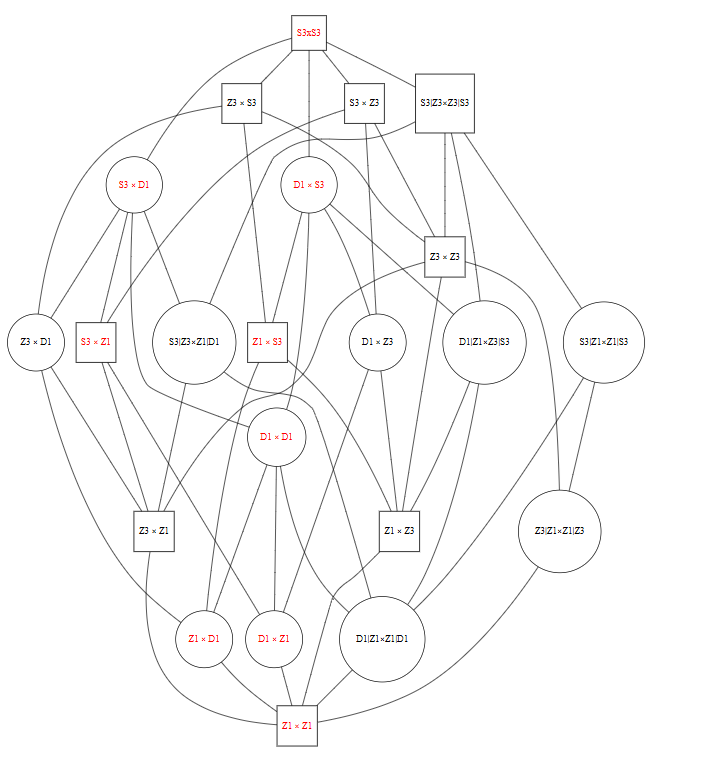}
    \caption{$S_3\times S_3$ Subgroup Lattice}
    \label{fig:lattice}
\end{figure}

\begin{table}[ht]
\centering
\renewcommand{\arraystretch}{0.7}
\begin{tabular}{lclc}
\hline
$K$ & $|K|$ & $W_G(K)$ & $|W_G(K)|$ \\
\hline
$\mathbb Z_1\times\mathbb Z_1$ & 1  & $S_3\times S_3$ & 36 \\
$\mathbb Z_1\times D_1$ & 2  & $S_3$ & 6 \\
$D_1\times\mathbb Z_1$ & 2  & $S_3$ & 6 \\
$D_1\times D_1$ & 4  & $\mathbb Z_1$   & 1 \\
$\mathbb Z_1\times S_3$ & 6  & $S_3$ & 6 \\
$S_3\times\mathbb Z_1$ & 6  & $S_3$ & 6 \\
$S_3\times D_1$ & 12 & $\mathbb Z_1$   & 1 \\
$D_1\times S_3$ & 12 & $\mathbb Z_1$   & 1 \\
$S_3\times S_3$ & 36 & $\mathbb Z_1$   & 1 \\
\hline
\end{tabular}
\caption{$W(K)$ in $S_3\times S_3$}
\label{tab:WeylS3}
\end{table}\vskip0.3cm
\begin{table}[ht]
\centering
\renewcommand{\arraystretch}{0.7}
\begin{tabular}{c|ccccccccc}
\hline
$K \backslash H$ & $\mathbb Z_1\times\mathbb Z_1$ & $\mathbb Z_1\times D_1$ & $D_1\times\mathbb Z_1$ & $D_1\times D_1$ & $\mathbb Z_1\times S_3$ & $S_3\times\mathbb Z_1$ & $S_3\times D_1$ & $D_1\times S_3$ & $S_3\times S_3$ \\
\hline
$\mathbb Z_1\times\mathbb Z_1$ & 1 & 3 & 3 & 9 & 1 & 1 & 3 & 3 & 1 \\
$\mathbb Z_1\times D_1$ & 0 & 1 & 0 & 3 & 1 & 0 & 1 & 3 & 1 \\
$D_1\times\mathbb Z_1$ & 0 & 0 & 1 & 3 & 0 & 1 & 3 & 1 & 1 \\
$D_1\times D_1$ & 0 & 0 & 0 & 1 & 0 & 0 & 1 & 1 & 1 \\
$\mathbb Z_1\times S_3$ & 0 & 0 & 0 & 0 & 1 & 0 & 0 & 3 & 1 \\
$S_3\times\mathbb Z_1$ & 0 & 0 & 0 & 0 & 0 & 1 & 3 & 0 & 1 \\
$S_3\times D_1$ & 0 & 0 & 0 & 0 & 0 & 0 & 1 & 0 & 1 \\
$D_1\times S_3$ & 0 & 0 & 0 & 0 & 0 & 0 & 0 & 1 & 1 \\
$S_3\times S_3$ & 0 & 0 & 0 & 0 & 0 & 0 & 0 & 0 & 1 \\
\hline
\end{tabular}
\caption{ The number of $n(K,H)$.}
\label{tab:nHK}
\end{table}\vskip0.3cm
Combining Tables \ref{tab:WeylS3}, \ref{tab:nHK} with \eqref{eqn:recur},
one has the orbit types and gradient degree coefficients listed in Table \ref{tab:dK-nK}.
\begin{table}[H]
\centering
\small
\setlength{\tabcolsep}{4pt}
\begin{tabular}{c|ccccccccc}
\hline
$K$ & $S_3\times S_3$ & $S_3\times D_1$ & $D_1\times S_3$ & $D_1\times D_1$ & $S_3\times\mathbb{Z}_1$ & $\mathbb{Z}_1\times S_3$ & $D_1\times\mathbb{Z}_1$ & $\mathbb{Z}_1\times D_1$ & $\mathbb{Z}_1\times\mathbb{Z}_1$\\
\hline
$d_K$ & $0$ & $1$ & $1$ & $0$ & $3$ & $3$ & $-2$ & $-2$ & $-12$\\
$|W_G(K)|$ & $1$ & $1$ & $1$ & $1$ & $6$ & $6$ & $6$ & $6$ & $36$\\\hline\hline
$n_K$ & $0$ & $1$ & $1$ & $-2$ & $0$ & $0$ & $0$ & $0$ & $0$\\
\hline
\end{tabular}
\caption{Fixed-point degrees $d_K$ and coefficients $n_K$ of the equivariant degree.}
\label{tab:dK-nK}
\end{table}
\noindent The gradient degree is given by 
\[
\nabla_G\text{-}\deg(\nabla\widetilde{\mathscr{L}},\Omega)=\textcolor{red}{(D_1\times S_3)}+\textcolor{red}{(S_3\times D_1)}-2\,(D_1\times D_1),
\]
and the symmetry index therefore is 
\begin{align*}
\mathcal{I}_G(\widetilde{\mathscr{L}},\Omega)
=
\Big\{
&\bigl((S_3\times D_1),1\bigr),
\bigl((D_1\times S_3),1\bigr)
\Big\}.
\end{align*}
Figure \ref{fig:transformer_symmetry} illustrates the matrix constraints imposed by the two detected symmetry types.
\begin{figure}[H]
    \centering
    \includegraphics[width=1\linewidth]{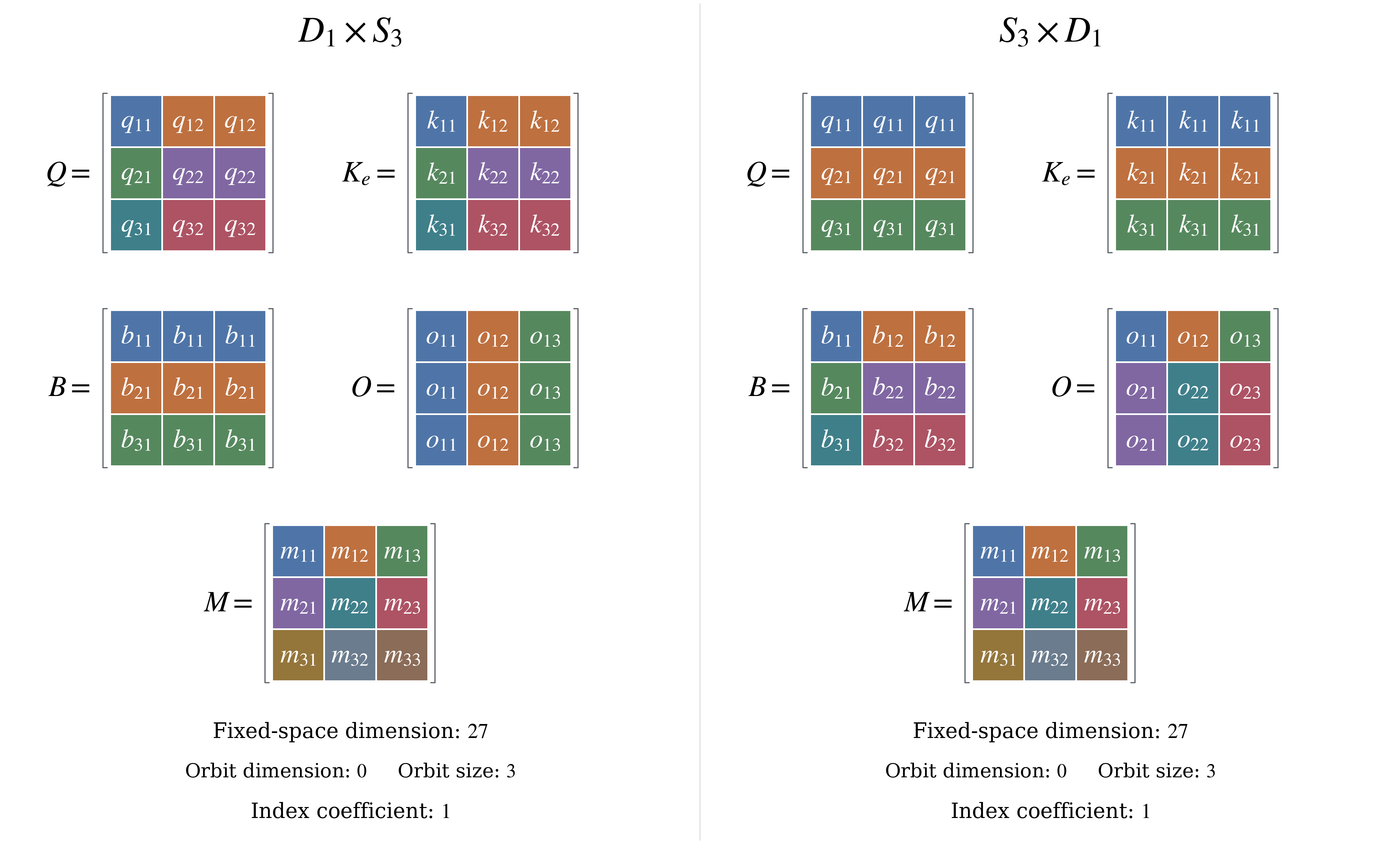}
    \caption{Matrix patterns in the fixed-point spaces \(V^{D_1\times S_3}\) and \(V^{S_3\times D_1}\)}
    \label{fig:transformer_symmetry}
\end{figure}

\vskip 0.6cm
\section{Future work}

In this work, we have developed topological symmetry index based on the family of Equivariant Degree Theories that yield topological invariants of the critical-point structure of loss functions for symmetric neural networks. The MLP, transformer, and ENN examples in Section \ref{sec:example} are toy models designed to illustrate the degree computations. Extending the methods to deeper MLPs, broader classes of ENNs, and multi-head, multi-layer transformers with layer normalization is a direction for future work.\\

Equivariant degree theory also offers tools beyond those used here, including \(S^1\)-equivariant and twisted equivariant degrees. Their potential applications to loss landscapes, particularly to continuous families of critical points and training dynamics, remain to be explored.\\

Finally, as mentioned in Section \ref{sec: terminology}, the equivariant degree framework assumes a compact Lie group. It therefore does not directly accommodate noncompact symmetries such as positive rescaling in ReLU networks or the attention symmetry \((Q,K)\mapsto(QA,KA^{-T})\), where \(A\) is invertible. Whether suitable normalized slices or quotient spaces can bring these symmetries within a compact-group framework remains an open question.\vskip0.6cm

\appendix
\section{Euler and Burnside Rings\label{app:Euler}} The additive structure of the Euler ring is $U(G)=\mathbb Z[\Phi(G)].$
Thus, its elements are finite integer linear combinations of conjugacy classes in \(\Phi(G)\). More precisely, every \(a\in U(G)\) has a unique expansion
\[
a=\sum_{(K)\in\Phi(G)}n_K(K),
\qquad n_K\in\mathbb Z,
\]
with only finitely many nonzero coefficients. We write
\[
\operatorname{coeff}^{K}(a):=n_K
\]
for the coefficient associated with \((K)\). Under the Euler-ring multiplication \(*\), the class \((G)\) serves as the identity:
\[
(G)*a=a
\qquad\text{for every }a\in U(G).
\]
\begin{lemma} Multiplication by an invertible element \(a\in U(G)\) preserves a nonzero coefficient at \((K)\) in the product \((K)*a\). In other words,
\[
\operatorname{coeff}^{K}\bigl((K)*a\bigr)\ne0
\qquad\text{for every }(K)\in\Phi(G).
\]
\end{lemma}
\begin{proof} The multiplication rule ensures that every class appearing with a nonzero coefficient in \((K)*a\) lies below or equals \((K)\). We may therefore write
\[
(K)*a=\sum_{(H)\le(K)}m_H(H),
\]
where the sum has finite support.
Suppose that \(m_K=0\). Then every class contributing to this expansion satisfies \((H)<(K)\). Multiplying once more by \(a^{-1}\) can only produce classes lying below these \((H)\), so
\[
\bigl((K)*a\bigr)*a^{-1}
=
\sum_{(S)<(K)}c_S(S).
\]
However, associativity and invertibility give
\[
\bigl((K)*a\bigr)*a^{-1}
=(K)*(a*a^{-1})
=(K)*(G)
=(K).
\]
This contradicts uniqueness of the expansion in the basis \(\Phi(G)\). Hence \(m_K\ne0\). 
\end{proof}\vskip0.3cm
\noindent To describe the Burnside ring, consider
\[
\Phi_0(G)=\{(K)\in\Phi(G):\dim W(K)=0\},
\qquad
A(G)=\mathbb Z[\Phi_0(G)].
\]
For generators belonging to \(\Phi_0(G)\), the Burnside product retains precisely the \(\Phi_0(G)\)-components of their Euler-ring product:
\[
(K)\cdot(H)
=
\sum_{(S)\in\Phi_0(G)}n_S(S).
\]
Its coefficients are
\begin{equation}\label{eqn:coe}
n_S
=
\chi\!\left((G/K\times G/H)_S/N(S)\right)
=
\left|(G/K\times G/H)_S/N(S)\right|,
\end{equation}
where \(\chi\) denotes the ordinary Euler characteristic. Extending this multiplication bilinearly equips \(A(G)\) with the ring structure called the Burnside ring of \(G\).
The coefficients in \eqref{eqn:coe} can be evaluated recursively, beginning with the larger conjugacy classes:
\[
n_S=
\frac{
n(S,K)|W(K)|\,n(S,H)|W(H)|
-
\displaystyle\sum_{(\widetilde S)>(S)}
n(S,\widetilde S)\,n_{\widetilde S}|W(\widetilde S)|
}{
|W(S)|
}.
\]
All classes appearing in this formula belong to \(\Phi_0(G)\). The quantities involved are defined using the transporter
\[
N_G(S,K):=\{g\in G:gSg^{-1}\subseteq K\}.
\]
With this convention, its left coset spaces are
\[
K\backslash N_G(S,K)
=
\{Kg:g\in N_G(S,K)\},
\]
and
\[
n(S,K)
=
\left|N_G(K)\backslash N_G(S,K)\right|,
\]
where \(N_G(K)\) is the normalizer of \(K\).
The nilpotence behavior of the basis elements depends on the dimension of their Weyl groups, as recorded in the following result: %see [11, Proposition 1.14, p. 231].
\begin{proposition}\label{prop:ha} Let $\Phi_n(G)=\{(K)\in\Phi(G):\dim W(K)=n\}.$ For \((K)\in\Phi_n(G)\), the following alternatives hold:
\begin{enumerate}
\item If \(n>0\), the element \((K)\) is nilpotent in \(U(G)\); that is,\[
   (K)^k=0
   \quad\text{for some positive integer }k.
   \]
\item If \(n=0\), every positive power remains nonzero:\[
   (K)^k\ne0
   \quad\text{for all positive integers }k.
   \]
\end{enumerate}
\end{proposition}
\begin{corollary} 
Every finite integer linear combination
\[
\alpha=\sum_{j=1}^{k}n_j(L_j),
\qquad
\dim W(L_j)>0\quad (j=1,\ldots,k),
\]
is nilpotent in \(U(G)\). Thus, \(\alpha^n=0\) for some positive integer \(n\).
\end{corollary}
\begin{proof} 
We proceed by induction on the number of summands. For a single summand, Proposition \ref{prop:ha} gives an integer \(q\ge1\) such that \((L_1)^q=0\), and consequently
\[
\bigl(n_1(L_1)\bigr)^q=n_1^q(L_1)^q=0.
\]
Assume the assertion holds for sums of \(k\) terms, and write a sum of \(k+1\) terms as
\[
\beta=\alpha+n_{k+1}(L_{k+1}),
\qquad
\alpha=\sum_{j=1}^{k}n_j(L_j).
\]
By the induction hypothesis and Proposition \ref{prop:ha}, there exist positive integers \(p,q\) satisfying
\[
\alpha^p=0,
\qquad
(L_{k+1})^q=0.
\]
Since \(U(G)\) is commutative, the binomial expansion yields
\[
\beta^m
=
\sum_{\ell=0}^{m}
\binom{m}{\ell}
\alpha^\ell
n_{k+1}^{\,m-\ell}
(L_{k+1})^{m-\ell}.
\]
Choose \(m\ge p+q\). In each summand, either \(\ell\ge p\), so that \(\alpha^\ell=0\), or \(\ell<p\), in which case \(m-\ell\ge q\) and \((L_{k+1})^{m-\ell}=0\). Every term therefore vanishes, proving that \(\beta^m=0\). 
\end{proof}
\vskip.6cm 

\section{$G$-Equivariant Brouwer Degree}\label{app:Brouwer}

In this section, we introduce the Equivariant Brouwer Degree, which provides the foundation for the gradient-related degree theories presented in sections \ref{sec:orthogonalrep} and \ref{sec:manifold}. Our exposition takes a computational approach and is structured around the various methods of computing the equivariant Brouwer degree. (For a more classical and comprehensive treatment of the theory, including its formal development and further theoretical details, we refer the readers to \cite{balanov2025degree}). We first introduce its definition and fundamental properties from the viewpoint of its general construction, together with a simplification by passing to a subgroup/ normal slice. We then present two concrete computational approaches: one based on reduction to regular normal maps, and the other based on linear \(G\)-isomorphisms and basic degrees. Before proceeding to the main discussion, we first introduce the setting as follows:\vskip 0.3cm
For each finite-dimensional orthogonal representation \(V\), let \(\mathcal M^G(V,V)\) contain the admissible pairs \((f,\Omega)\). Their union over all such representations is denoted by
\[
\mathcal M^G=\bigcup_V\mathcal M^G(V,V).
\]\vskip0.3cm
\begin{definition}[Fundamental Domain]
Suppose the action on \(V\) is free, and write \(\pi:V\to V/G\) for the quotient map. A submanifold \(D_0\) defines a fundamental domain \(D=\overline{D_0}\) when its distinct translates are disjoint, its closure meets every orbit, and its dimension equals that of the quotient:
\begin{align*}
&gD_0\cap hD_0=\varnothing\quad(g\neq h),\\
&\pi(\overline{D_0})=V/G,\\
&\dim D_0=\dim(V/G).
\end{align*}
\end{definition}
\noindent For an invariant open region \(\Omega\), each orbit-type subset \(\Omega_{(K)}\) is an invariant submanifold. At a point \(\theta\) of this stratum, the ambient representation splits into tangent and normal directions:
\[
V=T_\theta\Omega_{(K)}\oplus N_\theta\Omega_{(K)}.
\]
\begin{definition}[Normal map]
An admissible equivariant map \(f\) is normal in \(\Omega\) if it agrees locally with the identity in directions normal to the orbit-type stratum at each zero. Explicitly, for every \(\theta\in f^{-1}(0)\cap\Omega_{K}\), some \(\delta_\theta>0\) satisfies
\[
f(\theta+v)=v
\qquad
\bigl(v\in N_\theta\Omega_{(K)},\ \|v\|<\delta_\theta\bigr).
\]
\end{definition}
\begin{definition}[Regular normal map] A normal map is regular normal when it is \(C^1\) and the fixed-point restriction is regular at every zero of each occurring isotropy type. Thus, for each \((K)\in\Phi(G,f^{-1}(0)\cap\Omega)\), zero must be a regular value of $\left.f^K\right|_{\Omega_K},$
where $f^K=\left.f\right|_{V^K}.$
\end{definition}\vskip0.3cm
 
\subsection{Recursive-Construction Methodology (Based on Definition and Fundamental Properties)}\label{sec:rec}

\begin{definition}[Equivariant Brouwer Degree]
\label{thm:GpropDeg} The equivariant Brouwer degree $G\text{-}\deg:\mathcal M^G\longrightarrow A(G)$ associates a Burnside-ring element with every admissible pair, and the expansion is given by
\begin{equation}
\label{eq:G-deg0}
G\text{-}\deg(f,\Omega)
=
\sum_{(K_j)\in\Phi_0(G)} n_{j}(K_j)
=
n_{1}(K_1)+\cdots+n_{m}(K_m),
\end{equation}
with coefficients
\begin{equation}\label{eq:recurence}
n_{j}
=
\frac{
\deg\bigl(f^{K_j},\Omega^{K_j}\bigr)
-
\displaystyle\sum_{(H)>(K_j)}
n_H\,n(K_j,H)\,|W(H)|
}{
|W(K_j)|
},\quad j=1,\cdots, m.   
\end{equation}
Here $\deg\bigl(f^{K_j},\Omega^{K_j}\bigr)$ denotes the local Brouwer degree (the non-equivariant counterpart). One is referred to \cite{balanov2025degree} Chapter 2 for the details of local Brouwer degree and its computation.
\end{definition}\vs 
We next present the properties that uniquely determine the equivariant Brouwer degree, together with several additional properties relevant to its computation.
\begin{theorem}
The following three axioms determine \(G\text{-}\deg:\mathcal{M}^{G}\longrightarrow A(G)\) uniquely:
\begin{enumerate} \item[\textnormal{($\mathcal{P}_1$)}]
    \textbf{Additivity.}
    Let $\Omega_1,\Omega_2\subset\Omega$ be $G$-invariant disjoint open subsets and $f^{-1}(0)\cap\Omega\subset\Omega_1\cup\Omega_2.$ Then
    \[
    G\text{-}\deg(f,\Omega)
    =
    G\text{-}\deg(f,\Omega_1)
    +
    G\text{-}\deg(f,\Omega_2).
    \]
\item[\textnormal{($\mathcal{P}_2$)}]
    \textbf{Homotopy.}
    Let $h_t: V\longrightarrow V$ be an $\Omega$-admissible $G$-equivariant homotopy. Then
    $G\text{-}\deg(h_t,\Omega)$ is a constant on $t\in[0,1]$.

    \item[\textnormal{($\mathcal{P}_3$)}]
    \textbf{Normalization.}
    Suppose that $f$ is regular normal in $\Omega$ and that its zero set
    consists of a single $G$-orbit, i.e. $f^{-1}(0)\cap\Omega=G(x_0).$
Then \[
G\text{-}\deg(f,\Omega)=n_0(K)
\]
where $K=G_{x_0}$ and $n_0
    =
    \operatorname{sign}\det Df^K(x_0).$
\end{enumerate}
\end{theorem}
\begin{proposition} The axioms also yield the following consequences:
\begin{enumerate}
\item[\textnormal{($\mathcal{P}_4$)}]
\textbf{Existence.}
If, for $(K_j)\in\Phi_0(G)\;\; (j=1,\cdots,m)$, the corresponding coefficient $n_{j}\neq 0,$
then $f$ admits a zero $\theta_o$ in $\Omega$ whose isotropy group $(G_{\theta_o})\geq (K_j).$
\item[\textnormal{($\mathcal{P}_5$)}]
\textbf{Unitary.}
Let $\Omega\subset V$ be an open bounded $G$-invariant neighborhood
of the origin. Then 
\[
G\text{-}\deg(\operatorname{Id}_V,\Omega)=(G).
\]
\item[\textnormal{($\mathcal{P}_6$)}]
\textbf{Excision.}
Let open set $\Omega_0\subset\Omega$ be $G$-invariant and all zeros in $\Omega$ is a subset of $\Omega_0.$
Then
\[
G\text{-}\deg(f,\Omega)
=
G\text{-}\deg(f,\Omega_0).
\]

\item[\textnormal{($\mathcal{P}_7$)}]
\textbf{Product.}
For two admissible $G$-pairs $(f_1,\Omega_1)$ and $(f_2,\Omega_2)$, one has
\[
G\text{-}\deg
\bigl(f_1\times f_2,\Omega_1\times\Omega_2\bigr)
=
G\text{-}\deg(f_1,\Omega_1)
\cdot
G\text{-}\deg(f_2,\Omega_2).
\]
The product is taken in Burnside ring.

\item[\textnormal{($\mathcal{P}_8$)}]
\textbf{Suspension.}
Let
$U$ be an open bounded $G$-invariant neighborhood of the
origin of the orthogonal $G$-representation $W$. Then from unitary and product property, one has
\[
G\text{-}\deg
\bigl(f\times\operatorname{Id}_W,\Omega\times U\bigr)
=
G\text{-}\deg(f,\Omega).
\]
\end{enumerate}    
\end{proposition}\vskip0.3cm
\noindent These properties provide the main computational tools. Homotopy invariance permits a simpler representative, while nonzero coefficients certify zeros with corresponding isotropy. Additivity and excision isolate local contributions; product decompositions then allow calculations in terms of basic degrees. Section \ref{sec: rnm} explains the regular-normal construction behind the normalization axiom.\vskip0.3cm
We next localize the domain near an orbit and relate the calculation to its stabilizer. Given \(\theta_o\in V\), let \(\mathcal U\) be a neighborhood of \(G(\theta_o)\) such that \(f\) is \(\mathcal U\)-admissible. Our aim is to compute $G\text{-}\deg(f,\mathcal U). $
In practical applications, this computation can be reduced to a corresponding problem on a normal slice at \(\theta_o\), with the symmetry reduced to \(G_{\theta_o}\). To formulate this reduction, we first introduce the following ring homomorphism: 

To distinguish the ambient groups, write \((K)_{G'}\) and \((K)_G\) for the two conjugacy classes of \(K\leq G'\). The homomorphism \(\Theta:A(G')\to A(G)\) is specified on generators by 
$\Theta:A(G')\longrightarrow A(G)$
by 
\begin{equation}\label{eq:burnsideringhomomorhpism}
\Theta\bigl((K)_{G'}\bigr)
=
\begin{cases}
(K)_G, & \text{if } |W_G(K)|<\infty,\\[2mm]
0,   & \text{otherwise}.
\end{cases}
\end{equation}
\begin{theorem}
    Let $\theta_o\in V$ and 
$V_o:=T_{\theta_o}G(\theta_o)$ be
the tangent space to the orbit $G(\theta_o)$ at $\theta_o$. Define $S_o:=V_o^\perp$ the slice to $G(\theta_o)$ at $\theta_o$. Denote by \(P_o:V\to S_o\) the orthogonal projection. Choose \(\varepsilon>0\) sufficiently small that
\[
\mathcal U:=G\bigl(\theta_o+B_\varepsilon(S_o)\bigr)
\]
is a tubular neighborhood. Define
\[
f_o:S_o\to S_o,\qquad
f_o(v):=P_of(\theta_o+v),
\]
and assume that $f_o$ $B_\varepsilon(S_o)$ admissible, then
\[
G\text{-}\deg(f,\mathcal U)
=
\Theta\!\left(
G_{\theta_o}\text{-}\deg\bigl(f_o,B_\varepsilon(S_o)\bigr)
\right),
\]
where $\Theta$ is the homomorphism taken from \eqref{eq:burnsideringhomomorhpism}. 
\end{theorem}
\begin{remark}
The boundary condition holds for sufficiently small \(\varepsilon\) if \(0\) is an isolated zero of \(f_o\). If \(\dim G_{\theta_o}=\dim G\), then \(S_o=V\), \(P_o=\mathrm{Id}_V\), and \(f_o(v)=f(\theta_o+v)\)
\end{remark}
\vskip0.3cm
\subsection{Methodology Through Regular Normal Map}\label{sec: rnm} To compute $G\text{-}\deg(f,\Omega),$ we begin by fixing a bounded invariant open domain \(\Omega\) on which \(f\) is admissible. Based on the homotopy property introduced in ($\mathcal P_2$), the objective of this section is to reduce the degree computation for the original nonlinear map \(f\) to that of a simpler representative $f_o$. This raises the following several questions: What should such a representative be? How can one ensure the homotopy between \(f\) and $f_o$ \(\Omega\)-admissible? And why is $f_o$ a good representative and how can its degree be calculated? To address these questions, we need the following theorem (see \cite{balanov2025degree} for details):
\begin{theorem}
Given an admissible pair \((f,\Omega)\) and \(\varepsilon>0\), there exists an equivariant map \(f_o\), regular normal in \(\Omega\), whose uniform distance from \(f\) on \(\overline\Omega\) is less than \(\varepsilon\):
\[
\sup_{\theta\in\Omega}\|f(\theta)-f_o(\theta)\|<\varepsilon.
\]
\end{theorem}
\noindent Therefore, by considering the linear homotopy
\[
h(t,\theta)=t f_o(\theta)+(1-t)f(\theta), \quad t\in[0,1],\ \theta\in V, 
\]
\noindent one obtains a $G$ \(\Omega\)-admissible homotopy between the original map \(f\) and its regular normal approximation \(f_o\). By homotopy property ($\mathcal P_2$), $G\text{-}\deg(f,\Omega)=G\text{-}\deg(f_o,\Omega)$. Moreover, regular normal approximation make the zero set amenable to degree computation: Normality provides a canonical local structure near zeros of each orbit type, fundamental domains removes the redundancy arising from the \(G\)-action and reduces the problem to a local Brouwer degree computation. Regularity, together with the Inverse Function Theorem, further ensures that the zeros are isolated, allowing their local contributions to be computed individually and summed using additivity property ($\mathcal P_1$).\vskip0.3cm
We next present the analytic formula to compute the equivariant Brouwer degree.
\begin{proposition}[Analytic Formula]
Let $(f_o,\Omega)$ be an admissible $G$-pair, with $f_o$ regular normal in
$\Omega$. For each $(K)\in\Phi_0(G,\Omega)$, the zeros of the restriction
$f_o^K$ in $\Omega$ are isolated and can be represented by finitely many
$W(K)$-orbits,
\[
(f_o^K)^{-1}(0)\cap\Omega
=
W(K)(\theta_1)\cup\cdots\cup W(K)(\theta_n).
\]
Then 
\[
G\text{-}\deg(f_o,\Omega)
=
\sum_{(K)\in\Phi_0(G,V)} n_K(K),\qquad
n_K
=
\sum_{j=1}^{n}
\operatorname{sign}\det Df_o^K(\theta_j).
\]
\end{proposition}\vskip0.3cm
\subsection{Methodology Through Linear Isomorphism and Reduction to Basic Degree} 
Notice that both approaches discussed above may be difficult to carry out explicitly. The recursive construction in \ref{sec:rec} requires the computation of \(\deg(f^H,\Omega^H)\), while the regular-normal approach in \ref{sec: rnm} requires an explicit regular normal approximation and suitable fundamental domains. Therefore, in this section, we consider a linearization technique based on local or global linear approximation of the original nonlinear map, reducing the degree computation to that of a linear \(G\)-isomorphism. %Since \(V\) is an orthogonal \(G\)-representation, representation and spectral decompositions can then be applied, leading to a formulation in terms of basic degrees.
\vskip0.3cm
Consider $V$ as a $n$-dimensional orthogonal $G$-representation and let $L:V\to V$ be a $G$-equivariant linear isomorphism. 
\begin{definition}[$G$-Isotypic Decomposition]\label{def:isotypic} 
The finite dimensional $G$-representation $V$ admits a $G$-isotypic decomposition
\[
V=V_0\oplus V_1\oplus \cdots\oplus V_r,
\]
where each $V_i$ is modeled on the irreducible $G$-representation $\mathcal V_i,\;\; i=0,\cdots,r.$
\end{definition}\vskip0.3cm
\noindent Since $L$ is $G$-equivariant, then $L(V_i)\subseteq V_i.$ Now let $\delta (L)$ be the spectrum of $L.$ For $\beta\in \delta(L),$ denote by $E(\beta)$ the generalized eigenspace, then $E(\beta)$ is $G$-invariant. Define
$E_i(\beta):=E(\beta)\cap V_i,$
obviously, $L(E_i(\beta))=E_i(\beta).$ We further define
\[
E_i^-:=\bigoplus_{\beta\in \delta^-{(L)}}E_i(\beta),
\]
where $\delta^-{(L)}$ denotes negative spectrum. It follows
that $L(E_i^-)=E_i^-.$ Hence the restriction 
$L|_{E_i^-}:E_i^-\to E_i^-$
is a well-defined $G$-equivariant linear isomorphism. Notice that $L|_{E_i^-}$ is $G$-equivariantly homotopic to $-\operatorname{Id}|_{E_i^-},$ by defining $\mu_i(\beta)
=
\frac{\dim E_i(\beta)}{\dim \mathcal{V}_i}$ and 
$\mu_i
=
\sum_{\beta\in\delta^-(L)} \mu_i(\beta),$ we have
\[
G\text{-}\deg\left(
\left.L\right|_{E_i^-}, B(E_i^-)
\right)
=
\left(
G\text{-}\deg\left(
-\operatorname{Id}, B(\mathcal{V}_i)
\right)
\right)^{\mu_i}.
\]
This motivates the following definition of basic degree: 
\begin{definition}[Basic Degree] Define the $\cV_i$-basic degree (or simply {\it basic degree}) as following
\begin{equation} \label{eq:basicdeg}
\deg_{\cV_i} := G\text{-}\deg(-\id, B(\cV_i)).
\end{equation}
Moreover, the basic degree takes the form
\begin{equation}
\deg_{\mathcal{V} _{i}}=\sum_{(K)}n_{K}(K)    
\end{equation}
where
\begin{equation}\label{eq:formula}
n_{K}=\frac{(-1)^{\dim\mathcal{V} _{i}^{K}}- \sum_{(K)<(H)}{n_{H}\, n(K,H)\, \left|  W(H)\right|  }}{\left|  W(K)\right|  }.
\end{equation}
\end{definition}\vskip 0.3cm
\noindent Therefore, by homotopy ($\mathcal P_2$) and product (P7) properties, one has
\begin{equation}
    G\text{-}\deg(L,B(V))
=
\prod_{\beta\in\delta^-(L)}
\prod_{i=0}^{r}
\left(\deg_{\mathcal{V}_i}\right)^{\mu_i(\beta)}
=
\prod_{i=0}^{r}
\left(\deg_{\mathcal{V}_i}\right)^{\mu_i},
\end{equation}
where the product is taken in the Burnside ring $A(G).$
We next list some important properties of basic degrees.
\begin{definition}[Leading Coefficient]
Let $a=(G)+\sum_{j=1}^{m} n_j(K_j)\in A(G),$
where $n_j\neq 0,
\;
(K_j)<(G),
\; j=1,\ldots,m.$
Among the orbit types appearing with nonzero coefficients, suppose that
$(K_{j_o})$ is maximal with respect to the ordering of orbit types.
Then the coefficient $n_{j_o}$ associated with $(K_{j_o})$ is called
a \emph{leading coefficient} of $a$.
\end{definition}

\begin{lemma}\label{le:basic_coefficient}
Let $n_o$ denote the leading coefficient associated with the orbit type
$(K_o)$ in the basic degree $\deg_{\mathcal{V}_i}$. Then
$\dim \mathcal{V}_i^{K_o}$ is odd, and $n_o$ is determined as follows:
\[
n_o=
\begin{cases}
-2, & |W(K_o)|=1,\\[2mm]
-1, & |W(K_o)|=2.
\end{cases}
\]
\end{lemma}
\begin{lemma}\label{le:involutive}
Every basic degree $\deg_{\mathcal{V}_i}\in A(G)$ is involutive in Burnside ring. In particular,
\[
\left(\deg_{\mathcal{V}_i}\right)^2
=
\deg_{\mathcal{V}_i}\cdot\deg_{\mathcal{V}_i}
=
(G).
\]
\end{lemma}
\vskip.3cm

%---

\section{Amalgamated Groups}\label{app:aml}
We describe the subgroups of \(S_k\times S_k\) using a notation based on Goursat’s lemma \cite{Goursat}. More generally, a subgroup \(\mathcal K\leq G_1\times G_2\) can be represented through subgroups \(K\leq G_1\), \(H\leq G_2\), and surjective homomorphisms onto a common quotient \(S\):
\[
\mathcal K
=
\{(k,h)\in K\times H:\varphi(k)=\psi(h)\}.
\]
We denote this fiber-product subgroup by 
\[
K^{\varphi}\times_S^{\psi}H.
\]
This construction is referred to as an amalgamated subgroup. Set $K_0=\ker\varphi,$ $H_0=\ker\psi.$
Once the identification \(K/K_0\cong H/H_0\cong S\) has been fixed, we abbreviate the subgroup as
\[
\mathcal K=K^{K_0}\times_S^{H_0}H.
\]
The chosen quotient identification is understood whenever this abbreviated notation is used.

\section{Some computations}\label{app:computations}
\subsection{Some computations for softplus teacher-student model in \ref{sec:teacherstudent}}\label{app:computations_1}
\leavevmode\par
\smallskip
\noindent \textbf{An a priori bound.} We work with $k=4$, $\beta=1$, and $\lambda=0.01$.
For $x\sim\mathcal N(0,I_4)$, write
\[
\sigma(z)=\log(1+e^z),
\;\;
s(z)=\sigma'(z),
\]
and set \[
A=\sum_{i=1}^4\sigma(\theta_i^Tx),\;\;B=\sum_{j=1}^4\sigma(x_j),\;\;
D=\sum_{i=1}^4s(\theta_i^Tx)\theta_i^Tx.
\]
Let $m=W(e^{-1}),\;\;
C_B=2\sqrt{\mathbb E[B^2]},$ where $W$ denotes the principal real branch of the Lambert
$W$ function, defined by $W(z)e^{W(z)}=z$.

\begin{theorem}
\label{thm:softplus-apriori}
Let $t_o$ be the unique positive root of
\begin{equation}
\label{eq:softplus-radial-polynomial}
p(t)=\frac{t^4}{100}
-\left(8m\sqrt{\frac{2}{\pi}}+C_B\right)t
-16m\log 2.
\end{equation}
Then 
\[
\langle\nabla\mathscr L(\theta),\theta\rangle>0
\]
for every $\|\theta\|>t_o.$ Moreover, $t_o\approx9.6103<10$, and every critical point
of $\mathscr L$ lies in $B_{10}(0)$.
\end{theorem}

\begin{proof}
Put $t=\|\theta\|$ and $v=\sum_{i=1}^4\theta_i$.
Differentiating along the radial direction gives
\[
\langle\nabla\mathscr L(\theta),\theta\rangle
=\mathbb E[AD]-\mathbb E[BD]+\frac{t^4}{100}.
\]
Since
\[
\max_{a\ge0}\frac{a}{1+e^a}=m,
\]
we have $zs(z)\ge z-m$, and hence $D\ge v^Tx-4m$.
Gaussian integration by parts and the identity
$\mathbb E[s(\theta_i^Tx)]=\tfrac12$ yield
\[
\mathbb E[Av^Tx]
=\sum_{i=1}^4
\mathbb E[s(\theta_i^Tx)]\langle\theta_i,v\rangle
=\frac12\|v\|^2\ge0.
\]
Also, $\sigma(z)\le\log 2+|z|$ implies
\[
\mathbb E[A]
\le4\log 2+\sqrt{\frac{2}{\pi}}
\sum_{i=1}^4\|\theta_i\|
\le4\log 2+2\sqrt{\frac{2}{\pi}}\,t.
\]
Since $A\ge0$, it follows that
\[
\mathbb E[AD]
\ge-4m\,\mathbb E[A]
\ge-16m\log 2-8m\sqrt{\frac{2}{\pi}}\,t.
\]
On the other hand, $0<s<1$ gives
\[
\|D\|_{L^2}
\le\sum_{i=1}^4\|\theta_i^Tx\|_{L^2}
=\sum_{i=1}^4\|\theta_i\|
\le2t.
\]
Thus, by Cauchy--Schwarz,
\[
|\mathbb E[BD]|\le C_Bt,
\qquad
\langle\nabla\mathscr L(\theta),\theta\rangle\ge p(t).
\]
\noindent To bound $C_B$, use
\[
\sigma(z)
=\log 2+\frac z2+\log\cosh\frac z2
\le\log 2+\frac z2+\frac{z^2}{8}.
\]
Writing $S=\sum_{j=1}^4x_j$ and $Q=\sum_{j=1}^4x_j^2$,
we obtain
\[
\mathbb E[B^2]
\le
\mathbb E\!\left[
\left(4\log 2+\frac S2+\frac Q8\right)^2
\right]
=16(\log 2)^2+4\log 2+\frac{11}{8}<12,
\]
so $C_B<4\sqrt3<7$.
For numerical evaluation, independence gives
\[
\mathbb E[B^2]
=4\mathbb E[\sigma(Z)^2]
+12\bigl(\mathbb E[\sigma(Z)]\bigr)^2,
\qquad Z\sim\mathcal N(0,1).
\]
This yields $C_B\approx6.77694932$ and the stated
approximation of $t_o$. Finally, $p$ decreases and then increases to $+\infty$
on $[0,\infty)$. Since $p(0)<0$, it has exactly one
positive root and is positive thereafter.
Using $m<0.3$ and $C_B<7$, we also have
\[
p(10)>100-9.4\cdot10-4.8=1.2>0.
\]
\end{proof}
\subsection{Some computations for single-head single-layer transformer in \ref{subsec:fmtr-recursive}}\label{app:admissibility}
\leavevmode\par
\smallskip
\noindent \textbf{The auxiliary potential}.
Under the identification $E\cong\mathbb C$, the generators $(123)$
and $(23)$ act by multiplication by $e^{2\pi i/3}$ and complex
conjugation, respectively. Hence $q$, and therefore $\phi$, is
$S_3$-invariant. Writing $z=\zeta(v)$, we have
\[
\nabla q(z)=6\overline z^{\,2}(z^3-1),
\qquad
\nabla\phi(z)=\chi(q(z))\nabla q(z).
\]
Since $\chi(t)>0$ for $t>1/64$, critical set of $\phi$ is given by
\[
\{0\}\cup\mathcal A.
\]
The Hessian vanishes at the origin because the first nonconstant
term of $q$ is cubic, and on $\mathcal A$ by the flatness of $\chi$.
Thus
\begin{equation}
\nabla^2\phi=0
\qquad\text{on }\{0\}\cup\mathcal A.
\label{eq:potential-flatness}
\end{equation}

The disk $\{|\xi-1|\le 1/8\}$ avoids the branch value $0$ of
$z\mapsto z^3$, so its preimage $\mathcal A$ consists of three
disjoint closed topological disks. Moreover,
\begin{equation}
\frac12
<
\left(\frac78\right)^{1/3}
\le \|v\|
\le
\left(\frac98\right)^{1/3}
<
\frac54
\quad(v\in\mathcal A),
\qquad
\sqrt2\left(\frac78\right)^{1/3}>\frac54.
\label{eq:potential-radii}
\end{equation}

\noindent \textbf{Admissibility, homotopy, and degeneracy.} Consider
\[
\widetilde{\mathscr L}_t
=
t\mathscr L
+\frac{\lambda}{2}\|u\|^2
+\phi(v)+\phi(w),
\qquad 0\le t\le1.
\]
Using equation \eqref{eq:hahahah}, and the Hessian lower bound defining \(\mathcal C\) in \eqref{eq:lowerbound}, for $\|(v,w,u)\|\le5/4$ we obtain
\[
\begin{aligned}
\bigl\langle\nabla_u\widetilde{\mathscr L}_t(v,w,u),u\bigr\rangle
&=
\lambda\|u\|^2
+t\int_0^1
\bigl\langle
\nabla^2_{uu}\mathscr L(v,w,su)u,u
\bigr\rangle\,ds\\
&\ge
(\lambda-\mathcal C)\|u\|^2.
\end{aligned}
\]
Thus every critical point in this ball has $u=0$, and the remaining
equations reduce to
\[
\nabla\phi(v)=0,
\qquad
\nabla\phi(w)=0.
\]
The bounds in \eqref{eq:potential-radii} show that, for every
$t\in[0,1]$, the critical set in 
$\overline\Omega$
is the critical family in equation \eqref{eq:critical-family}, which lies entirely in $\Omega$.
Hence $\nabla\widetilde{\mathscr L}_t$ is an admissible $G$-gradient homotopy
on $\Omega$, and
\[
\nabla_G\text{-}\deg(\nabla\widetilde{\mathscr L},\Omega)
=
\nabla_G\text{-}\deg(\nabla\widetilde{\mathscr L}_0,\Omega).
\]

By \eqref{eq:hahahah}, the Hessian of $\mathscr L$ has zero blocks involving the two
$E$-coordinates at $u=0$. Its normal block is positive semidefinite
because $\mathscr L\ge0$ and
$\mathscr L(v,w,0)=0$.
Together with \eqref{eq:potential-flatness}, this gives, at every
critical point in $\Omega$,
\[
\nabla^2\mathscr{\widetilde{L}}
=
0_{E\oplus E}
\oplus
\left(
\lambda \mathbf{Id}+\nabla^2_{uu}\mathscr L
\right),
\qquad
\ker\nabla^2\mathscr{\widetilde{L}}
=
E\oplus E\oplus\{0\}.
\]
The Hessian kernel therefore has dimension four. By \eqref{eq:critical-family}, every critical point in \(\Omega\) is accumulated by other critical points. Since \(G\) is finite, no critical orbit in \(\Omega\) is isolated.\vskip1cm

\bibliographystyle{plain} 
\bibliography{mybib_minima}
\end{document}